\documentclass[11pt,twoside]{article}

\usepackage{fullpage}

\usepackage{epsf}
\usepackage{fancyhdr}
\usepackage{graphics}
\usepackage{graphicx}
\usepackage{psfrag}
\usepackage{microtype}
\usepackage{subfigure}
\usepackage{algorithmic}

\usepackage[linesnumbered,ruled]{algorithm2e}% http://ctan.org/pkg/algorithm2e
\DontPrintSemicolon
\usepackage{color}

\usepackage{amsthm}
\usepackage{amsfonts}
\usepackage{amsmath}
\usepackage{amssymb,bbm}
\usepackage{wrapfig}
\usepackage{tikz}
\usetikzlibrary{positioning}

\usepackage{natbib}

\usepackage{url}% for url's in bib
\usepackage[colorlinks,linkcolor=magenta,citecolor=blue, pagebackref=true,backref=true]{hyperref}
\renewcommand*{\backrefalt}[4]{%
    \ifcase #1 \footnotesize{(Not cited.)}%
    \or        \footnotesize{(Cited on page~#2.)}%
    \else      \footnotesize{(Cited on pages~#2.)}%
    \fi}

\long\def\comment#1{}
\usepackage{nicefrac}
\usepackage[normalem]{ulem}
\usepackage{chngpage}

 \usepackage{tabularx}%

\usepackage{enumitem}
\usepackage{booktabs}
\usepackage{caption}

\usepackage{mathtools}

\usepackage{fullpage}

\newtheorem{theorem}{Theorem}[section]
\newtheorem{corollary}[theorem]{Corollary}
\newtheorem{lemma}[theorem]{Lemma}
\newtheorem{proposition}[theorem]{Proposition}

\newtheorem{definition}{Definition}
\newtheorem{example}{Example}

\newtheorem{remark}[theorem]{Remark}

\newcommand{\EE}{\mathbb{E}}
\newcommand{\PP}{\mathbb{P}}

\newcommand{\st}{\textnormal{s.t.}}

\newcommand{\argmin}{\mathop{\rm argmin}}
\newcommand{\argmax}{\mathop{\rm argmax}}
\newcommand{\LCal}{\mathcal{L}}

\newcommand{\DCal}{\mathcal{D}}

\newcommand{\br}{\mathbb{R}}

\newcommand{\ba}{\begin{array}}
\newcommand{\ea}{\end{array}}

\newcommand{\FCal}{\mathcal{F}}

\newcommand{\Rspace}{\mathbb{R}}
\newcommand{\one}{\textbf{1}}
\newcommand{\zero}{\textbf{0}}
\newcommand{\bigO}{O}
\newcommand{\bigOtil}{\widetilde{O}}
\newcommand{\mydefn}{:=}

\begin{document}

%%%%%%% TITLE PAGE %%%%%%%%%%%%%%%%%%%%%%%%%%%%%%%%%%%%%%%%%%%%%%%%%%%

\begin{center}

{\bf{\LARGE{On the Complexity of Approximating \\[.2cm] Multimarginal Optimal Transport}}}

\vspace*{.2in}
{\large{
\begin{tabular}{cccc}
Tianyi Lin$^{\star, \diamond}$ & Nhat Ho$^{\star, \ddagger}$ &  Marco Cuturi$^{\triangleleft, \triangleright}$ & Michael I. Jordan$^{\diamond, \dagger}$ \\
 \end{tabular}
}}

\vspace*{.2in}

\begin{tabular}{c}
Department of Electrical Engineering and Computer Sciences$^\diamond$ \\
Department of Statistics$^\dagger$ \\ 
University of California, Berkeley \\
Department of Statistics and Data Sciences, University of Texas, Austin$^\ddagger$ \\ 
CREST - ENSAE$^\triangleleft$, Google Brain$^\triangleright$
\end{tabular}

\vspace*{.2in}

\today

\vspace*{.2in}

\begin{abstract} 
We study the complexity of approximating the multimarginal optimal transport (MOT) distance, a generalization of the classical optimal transport distance, considered here between $m$ discrete probability distributions supported each on $n$ support points. First, we show that the standard linear programming (LP) representation of the MOT problem is not a minimum-cost flow problem when $m \geq 3$. This negative result implies that some combinatorial algorithms, e.g., network simplex method, are not suitable for approximating the MOT problem, while the worst-case complexity bound for the deterministic interior-point algorithm remains a quantity of $\bigOtil(n^{3m})$. We then propose two simple and \textit{deterministic} algorithms for approximating the MOT problem. The first algorithm, which we refer to as \textit{multimarginal Sinkhorn} algorithm, is a provably efficient multimarginal generalization of the Sinkhorn algorithm. We show that it achieves a complexity bound of $\bigOtil(m^3n^m\varepsilon^{-2})$ for a tolerance $\varepsilon \in (0, 1)$. This provides a first \textit{near-linear time} complexity bound guarantee for approximating the MOT problem and matches the best known complexity bound for the Sinkhorn algorithm in the classical OT setting when $m = 2$. The second algorithm, which we refer to as \textit{accelerated multimarginal Sinkhorn} algorithm, achieves the acceleration by incorporating an estimate sequence and the complexity bound is $\bigOtil(m^3n^{m+1/3}\varepsilon^{-4/3})$. This bound is better than that of the first algorithm in terms of $1/\varepsilon$, and accelerated alternating minimization algorithm~\citep{Tupitsa-2020-Multimarginal} in terms of $n$. Finally, we compare our new algorithms with the commercial LP solver \textsc{Gurobi}. Preliminary results on synthetic data and real images demonstrate the effectiveness and efficiency of our algorithms. 
\end{abstract}

\let\thefootnote\relax\footnotetext{$^\star$ Tianyi Lin and Nhat Ho contributed equally to this work.}
\end{center}

%!TEX root = paper.tex
\section{Introduction}
The multimarginal optimal transport (MOT)~\citep{Gangbo-1998-MOT, Pass-2015-Multi}, the general problem of aligning or correlating $m \geq 2$ probability measures so as to maximize efficiency (with respect to a given cost function), is a generalization of the optimal transport (OT) problem~\citep{Villani-2003-Topic}. From the Kantorovich formulation~\citep{Kantorovich-1942-Translocation}, we seek to solve the following optimization problem, 
\begin{equation}\label{def:MOT}
\min_{\gamma \in \Pi(\mu_1, \mu_2, \ldots, \mu_m)}\int_{M_1 \times M_2 \times \cdots \times M_m} c(x_1, x_2, \ldots, x_m) \; d\gamma(x_1, x_2, \ldots, x_m), 
\end{equation}
where $\Pi(\mu_1, \mu_2, \ldots, \mu_m)$ denotes the set of positive joint measures on the product space $M_1 \times M_2 \times \ldots \times M_m$ whose marginals are $\{\mu_i\}_{i \in [m]}$, and $c(\cdot)$ is a given cost function. In the discrete setting where each of $\mu_i$ is supported on $n$ support points, the MOT problem is equivalent to a linear programming (LP) problem with $mn$ constraints and $n^m$ variables, which means that any algorithm requires at least $n^m$ arithmetic operations in general. 

The MOT problem has been recognized as the backbone of numerous important applications, such as matching in economics~\citep{Ekeland-2005-Optimal, Carlier-2010-Matching, Chiapporri-2010-Hedonic}, density functional theory in physics~\citep{Seidl-2007-Strictly, Buttazzo-2010-Optimal, Cotar-2013-Density, Mendl-2013-Kantorovich}, generalized Euler flow in fluid dynamics~\citep{Brenier-1989-Least, Brenier-1999-Minimal, Brenier-2008-Generalized} and financial mathematics~\citep{Dolinsky-2014-Robust, Galichon-2014-Stochastic}. Over the past five years, the MOT problem has begun to attract considerable attention, due in part to a wide variety of emerging applications in machine learning, including generative adversarial networks (GANs)~\citep{Choi-2018-Stargan, Cao-2019-Multi}, clustering~\citep{Mi-2020-Multi}, domain adaptation~\citep{Hui-2018-Unsupervised, He-2019-Attgan} and Wasserstein barycenters~\citep{Agueh-2011-Barycenters, Cuturi-2014-Fast, Benamou-2015-Iterative, Carlier-2015-Numerical, Srivastava-2018-Scalable}. Due to the space limit, we refer the interested readers to~\citet{Pass-2015-Multi} for other applications of the MOT problem and~\citet{Peyre-2019-Computational} for more details of the MOT problem from a computational point of view. 

In order to further motivate the MOT problem, we briefly describe two representative application problems arising from machine learning.
\begin{example} 
The multimarginal Wasserstein GANs~\citep{Cao-2019-Multi} are proposed to optimize a feasible MOT distance among different domains. This approach is based on a new dual formulation of the MOT distance and overcomes the limitations of existing methods by alleviating the distribution mismatching issue and exploiting cross-domain correlations. 

We consider $m \geq 2$ target domains $\{\DCal_k\}_{k \in [m]}$ and the associated generative models $g_k$ parameterized by $\theta_k$ for all $k \in [m]$. Let $\FCal = \{f: \br^d \rightarrow \br\}$ be the class of discriminators parameterized by $w$, we define the MOT distance in the dual form as follows, 
\begin{equation*}
W(\hat{\PP}_s, \hat{\PP}_{\theta_1}, \ldots, \hat{\PP}_{\theta_m}) = \max_f \EE_{x \sim \hat{\PP}_s} [f(x)] - \sum_{k=1}^m \lambda_k^+ \EE_{x \sim \hat{\PP}_{\theta_k}} [f(x)], \quad \st \ \hat{\PP}_{\theta_k} \in \DCal_k, f \in \Omega,    
\end{equation*}
where $\hat{\PP}_s$ is the real source distribution, $\hat{\PP}_{\theta_k}$ is the distribution generated by $g_k$ for all $k \in [m]$, and $\Omega = \{f \in \FCal \mid f(x) - \sum_{k=1}^m \lambda_k^+ f(\hat{x}^{(k)}) \leq c(x, \hat{x}^{(1)}, \ldots, \hat{x}^{(m)})\}$ where $x \in \hat{\PP}_s$ and $\hat{x}^{(k)} \in \hat{\PP}_{\theta_k}$ for all $k \in [m]$ are samples. Note that $\lambda_k^+$ reflects the importance of the $k$-th target domain and is set as $1/m$ in practice when no prior knowledge is available. 
\end{example}
\begin{example}
The free-support Wasserstein barycenter~\citep{Agueh-2011-Barycenters} is defined as a weighted barycenter of input measures $\{\mu_k\}_{k \in [m]}$ defined on $\br^d$ according to the OT distance. As shown by~\citet{Agueh-2011-Barycenters}, the computation of barycenters of measures can be computed by solving a MOT problem.

We consider the discrete setting where input measures are $\mu_k = \sum_{i=1}^n p_{k, i}\delta_{x_i}$ with weights $p_k=(p_{k,1}, \ldots, p_{kn}) \in \Delta^n$, the support points $\{x_i\}_{i \in [n]} \subseteq \br^d$ and the Dirac measure $\delta$. It is shown in~\citep{Agueh-2011-Barycenters} that the Wasserstein barycenter of $\{\mu_k\}_{k \in [m]}$ with weights $\lambda = (\lambda_1, \ldots, \lambda_m) \in \Delta^m$ according to the OT distance with the quadratic Euclidean distance ground cost function $c=\|\cdot\|^2$ is
\begin{equation*}
\mu_\lambda \mydefn \sum_{1 \leq i_k \leq n, \forall k \in [m]} \gamma_{i_1, \ldots, i_m}\delta_{A_{i_1, \ldots, i_m}(x)}, 
\end{equation*} 
where $A_{i_1, \ldots, i_m}(x) = \sum_{k=1}^m \lambda_k x_{i_k}$ is the Euclidean barycenter and $\gamma \in \br^{n \times \cdots \times n}$ is an optimal multimarginal transportation plan that solves the MOT problem in the LP form of
\begin{equation*}
\min_{X \in \br^{n \times \cdots \times n}} \langle C, X\rangle, \quad \st \sum_{1 \leq i_l \leq n, l \neq k, \forall l \in [m]} X_{i_1, \ldots, i_{k-1}, j, i_{k+1}, \ldots, i_m} = p_{kj} \textnormal{ for all } (k, j) \in [m] \times [n],   
\end{equation*}
where $C$ is defined as $C_{i_1, \ldots, i_m} = \sum_{k=1}^m (\lambda_k/2)\|x_{i_k} - A_{i_1, \ldots, i_m}(x)\|^2$ for $(i_1, \ldots, i_m) \in [n] \times \ldots \times [n]$. In practice, we set $\lambda_k = 1/m$ for all $k \in [m]$ when no prior knowledge is available.  

It is worthy noting that the barycenter $\mu_\lambda$ is in general composed of more than $m$ Diracs, and that these Diracs are not constrained to be on the support points $\{x_i\}_{i \in [n]}$. This is different from the fixed-support Wasserstein barycenter that must be on the same support points $\{x_i\}_{i \in [n]}$ of the input measures. To be specific, the free-support Wasserstein barycenter is the ``true" barycenter of measures, while the fixed-support Wasserstein barycenter is an approximation on the fixed support points. But, on the flip side of the coin, the fixed-support Wasserstein barycenter can be computed without solving any MOT problem and the complexity bound is polynomial in $m$, $n$ and $1/\varepsilon$~\citep{Kroshnin-2019-Complexity, Lin-2020-Revisiting} where $\varepsilon$ is the desired accuracy. 
\end{example}
\paragraph{Algorithms for the OT problem.} The OT problem is a special instance of the MOT problem with $m=2$ and has been studied thoroughly during the past decade. To the best of our knowledge, there are mainly two group of algorithms for solving the OT problem. 

The first line of algorithms are combinatorial graph algorithms~\citep{Klein-1967-Primal, Edmonds-1972-Theoretical, Hassin-1983-Minimum, Tardos-1985-Strongly, Galil-1988-Min, Goldberg-1990-Finding, Hassin-1992-Algorithms, Ervolina-1993-Two, Ervolina-1993-Canceling, Orlin-1993-Faster, Orlin-1997-Polynomial, Goldberg-1998-Beyond}. Indeed, the OT problem is a minimum-cost flow problem~\citep{Schrijver-2003-Combinatorial}, which has graph structure and leads to efficient combinatorial algorithms mentioned before. Examples include the primal-dual cost scaling algorithm~\citep{Goldberg-1990-Finding} and the network simplex algorithm~\citep{Orlin-1997-Polynomial}; see also~\citet{Daitch-2008-Faster} and~\citet{Lee-2014-Path} for some recent progresses.

The second line of algorithms, initialized with the Sinkhorn algorithm~\citep{Cuturi-2013-Sinkhorn}, are developed for solving the OT problem through either entropy regularization or continuous optimization algorithmic frameworks. Examples include Greenkhorn algorithm~\citep{Altschuler-2017-Near, Lin-2019-Efficient}, accelerated first-order primal-dual algorithms~\citep{Dvurechensky-2018-Computational}, accelerated Sinkhorn algorithms~\citep{Lin-2019-Efficiency, Guminov-2019-Accelerated}, and some other algorithms~\citep{Blanchet-2018-Towards, Jambulapati-2019-Direct, Lahn-2019-Graph, Xie-2020-Fast}. Even though these algorithms are very efficient, with easy to implement routines in practice, the Sinkhorn algorithm and its accelerated variants remain as the state-of-the-art approach for the OT problem and serve as the default solver in the celebrated \textsc{POT} package~\citep{Flamary-2017-Pot}.

\paragraph{Algorithms for the MOT problem.} While the theory for computing the OT distance has received ample attention, the theory for computing the MOT distance is still nascent. Since the MOT problem has the LP representation with $mn$ constraints and $n^m$ variables, it can be solved by many LP algorithms, e.g., the interior-point algorithm, whose complexity bounds are however not near-linear. That is to say, the dependence of $n$ is linear in $n^m$ up to the logarithmic factors.

Two specialized algorithms are avaliable for solving the MOT problem: multimarginal Sinkhorn algorithm and accelerated alternating minimization algorithm. The former one generalizes the Sinkhorn algorithm to the MOT setting but only has the asymptotic convergence analysis~\citep{Benamou-2015-Iterative, Benamou-2019-Generalized, Peyre-2019-Computational}; the latter one is proposed by the concurrent work~\citep{Tupitsa-2020-Multimarginal} for solving the same dual entropic regularized MOT problem as ours and achieves the complexity bound of $\bigOtil(m^3n^{m+1/2}\varepsilon^{-1})$ when applied to solve the MOT problem along with our rounding scheme. However, their algorithm is not a near-linear time approximation algorithm and the dependence of $n$ can be potentially improved.

\paragraph{Contribution:} In this paper, we study the complexity of approximating the MOT problem between $m$ discrete probability distributions with at most $n$ points in their respective supports. Our contributions can be summarized as follows: 
\begin{enumerate}
\item We show that the standard LP representation of the MOT problem is not a minimum-cost flow problem when $m \geq 3$. This implies the inefficiency of many combinatorial algorithms, including network simplex method, as well as the worst-case complexity bound of $\bigOtil(n^{3m})$ for the standard \textit{deterministic} interior-point algorithms.

\item We propose two simple and \textit{deterministic} algorithms for solving the entropic regularized MOT problem. The first one is named as \emph{multimarginal Sinkhorn algorithm} which can be also used to solve the MOT problem along with a new rounding scheme. The achieved complexity bound is $\bigOtil(m^3n^m\varepsilon^{-2})$, which is near-linear in terms of $n^m$, demonstrating that our algorithm is unimprovable in terms of $n$ in general setting. To the best of our knowledge, this is a first \textit{near-linear time} approximation algorithm for solving the MOT problem while the existing ones are either only shown convergent~\citep{Benamou-2015-Iterative, Benamou-2019-Generalized, Peyre-2019-Computational} or not near-linear time~\citep{Tupitsa-2020-Multimarginal}. The second one is named as \emph{accelerated multimarginal Sinkhorn algorithm} and achieves the complexity bound of $\bigOtil(m^3n^{m+1/3}\varepsilon^{-4/3})$ when applied to solve the MOT problem. This complexity bound is better than that of the first algorithm in terms of $1/\varepsilon$, and the accelerated alternating minimization algorithm~\citep{Tupitsa-2020-Multimarginal} in terms of $n$. 

\item We compare our algorithms with the commercial LP solver \textsc{Gurobi}. Preliminary results on both synthetic data and real images demonstrate the effectiveness and efficiency of our algorithms in practice. 
\end{enumerate} 
\paragraph{Organization.} The remainder of the paper is organized as follows. In Section~\ref{sec:prelim}, we present the background materials on the MOT problem and derive some important properties of the objective function in the dual entropic regularized MOT problem. In Section~\ref{sec:hardness}, we show that the standard LP representation of the MOT problem is not a minimum-cost flow problem when $m \geq 3$. In Sections~\ref{sec:sinkhorn} and~\ref{sec:acceleration}, we propose the multimarginal Sinkhorn and accelerated multimarginal Sinkhorn algorithms for solving the entropic regularized MOT problem. We also demonstrate that these algorithms can solve the MOT problem efficiently along with our new rounding scheme. In Section~\ref{sec:experiments}, we present some numerical results which validate the efficiency of our algorithms. We finally conclude this paper in Section~\ref{sec:conclusion}.
  
\paragraph{Notation.} We let $[n]$ be the set $\{1, 2, \ldots, n\}$ and $\Rspace^n_+$ be the set of all vectors in $\Rspace^n$ with non-negative components. $\one_n \in \Rspace^n$ refers to a vector with all of its components are $1$ and $\Delta^n$ is denoted as the probability simplex in $\Rspace^n_+$:  $\Delta^n = \{u \in \Rspace^n_+: \one_n^\top u = 1\}$. For a set $S$, we denote $|S|$ as its cardinality. For a differentiable function $f$, we denote $\nabla f$ and $\nabla_\beta f$ as the full gradient of $f$ and the gradient of $f$ with respect to $\beta$. For a vector $x \in \Rspace^n$ and $1 \leq p \leq \infty$, we denote $\|x\|_p$ as its $\ell_p$-norm and $\|x\|$ as its $\ell_2$-norm for simplicity. Let $x$ and $y$ be two vectors of same dimension, we denote $\min\{x, y\}$ as the component-wise minimum of $x$ and $y$. For a tensor $A = (A_{i_1, \ldots, i_m}) \in \Rspace^{n_1 \times \ldots \times n_m}$, we write $\|A\|_\infty = \max_{1 \leq i_k \leq n_k, \forall k \in [m]} |A_{i_1, \ldots, i_m}|$ and $\|A\|_1 = \sum_{1 \leq i_k \leq n_k, \forall k \in [m]} |A_{i_1, \ldots, i_m}|$, and denote $r_k(A) \in \br^{n_k}$ as its $k$-th marginal for $k \in [m]$ and each component is defined by 
\begin{equation*}
[r_k(A)]_j \mydefn \sum_{1 \leq i_l \leq n_l, \forall l \neq k} A_{i_1, \ldots, i_{k-1}, j, i_{k+1}, \ldots, i_m}.
\end{equation*}
Let $A$ and $B$ be two tensors of same dimension, we denote their Frobenius inner product as 
\begin{equation*}
\left\langle A, B\right\rangle \mydefn \sum_{1 \leq i_k \leq n_k, \forall k \in [m]} A_{i_1, \ldots, i_m} B_{i_1, \ldots, i_m}.
\end{equation*}
Given the dimension $n$ and accuracy $\varepsilon$, the notation $a = \bigO\left(b(n,\varepsilon)\right)$ stands for the upper bound $a \leq C \cdot b(n, \varepsilon)$ where $C>0$ is independent of $n$ and $\varepsilon$, and the notation $a = \bigOtil(b(n, \varepsilon))$ indicates the previous inequality where $C$ depends on the logarithmic function of $n$ and $\varepsilon$. 

%!TEX root = paper.tex
\section{Preliminaries}\label{sec:prelim}
In this section, we first present the linear programming (LP) representation of the multimarginal optimal transport (MOT) problem as well as a formal specification of an approximate multimarginal transportation plan. Then, we describe the entropic regularized MOT problem and derive the dual entropic regularized MOT problem where the objective function is in the form of the logarithm of sum of exponents. Finally, we provide several properties of this function which are useful for the subsequent analysis. 

\subsection{Linear programming representation}
The linear programming representation of the OT problem between two discrete probability distributions with $n$ supports dates back to the seminar work by~\citet{Kantorovich-1942-Translocation}, and can be written as  
\begin{eqnarray*}
\min\limits_{X \in \br^{n \times n}} \langle C, X\rangle \quad \st \ X\one_n = r, \ X^\top\one_n = c, \ X \geq 0. 
\end{eqnarray*}
In the above formulation, $X \in \br_+^{n \times n}$ denotes a \textit{transportation plan}, $C \in \br_+^{n \times n}$ denotes an \textit{nonnegative cost matrix}, and $r$ and $c$ stand for two probability distributions lying in the simplex $\Delta^n$. Approximately solving the OT problem amounts to finding an $\varepsilon$-approximate transportation plan $\hat{X}$ such that $\hat{X}\one_n = r$, $\hat{X}^\top\one_n = c$ and the following inequality holds true, 
\begin{equation*}
\langle C, \hat{X}\rangle \leq \langle C, X^\star\rangle + \varepsilon.
\end{equation*}
where $X^\star$ is defined as an optimal transportation plan of the OT problem.

As a straightforward generalization of the OT problem, the MOT problem is also a LP. Indeed, the problem of computing the MOT distance between $m \geq 2$ discrete probability distributions with $n$ supports is in the following form of 
\begin{equation}\label{prob:MOT}
\min_{X \in \br^{n \times \cdots \times n}} \langle C, X\rangle, \quad \st \ r_k(X) = r_k \textnormal{ for any } k \in [m], \ X \geq 0. 
\end{equation}
In the above formulation, $X$ denotes the \emph{multimarginal transportation plan}, $C \in \br_+^{n \times \cdots \times n}$ denotes a \emph{nonnegative cost tensor}, and $\{r_k\}_{k \in [m]}$ stand for a set of probability distributions all lying in $\Delta^n$. 

We see from Eq.~\eqref{prob:MOT}, that the MOT problem is a linear programming with $mn$ equality constraints and $n^m$ variables. The solution we hope to achieve is an $\varepsilon$-approximate multimarginal transportation plan which generalizes the notion of an $\varepsilon$-approximate transportation plan of the OT problem. More specifically, we have the following definition of $\varepsilon$-approximate multimarginal transportation plan.
\begin{definition}\label{def:MOT_plan}
The nonnegative tensor $\widehat{X} \in \br_+^{n \times \cdots \times n}$ is called an \emph{$\varepsilon$-approximate multimarginal transportation plan} if $r_k(\widehat{X}) = r_k$ for any $k \in [m]$ and the following inequality holds true, 
\begin{equation*}
\langle C, \widehat{X}\rangle \leq \langle C, X^\star\rangle + \varepsilon,
\end{equation*}
where $X^\star$ is defined as an optimal multimarginal transportation plan of the MOT problem. 
\end{definition}
With this definition in mind, one of the goals of this paper is to develop \textit{near-linear time approximation} algorithms for solving the MOT problem. In particular, we seek the algorithms whose running time required to obtain an $\varepsilon$-approximate multimarginal transportation plan is nearly linear in the number of unknown variables $n^m$. These algorithms are favorable in modern machine learning applications since they are unimprovable up to the logarithmic factors in general. Indeed, for the general MOT problem, the tensor $X \in \br_+^{n \times \cdots \times n}$ has $n^m$ unknown entries. In order to solve the MOT problem, the number of arithmetic operations required by any algorithms is at least $n^m$. 

In the classical OT setting,~\citet{Altschuler-2017-Near} has shown that the Sinkhorn algorithm is near-linear time approximation algorithm.~\citet{Benamou-2015-Iterative, Benamou-2019-Generalized} generalized the Sinkhorn algorithm to the MOT setting but did not provide any complexity bound guarantee for their algorithms. Thus, it is still unclear whether there exists a near-linear time approximation algorithm for the general MOT problem.

%%%%%%%%%%%%%%%%%%%%%%%%%%%%%%%%%%%%%%%%%%%%%%%%%%%%%%%%%%%%%%%%%%%%%%%%%%%%%%%
\subsection{Entropic regularized MOT and its dual form}
Building on Cuturi's entropic approach to the classical OT problem~\citep{Cuturi-2013-Sinkhorn}, we consider a regularized version of the MOT problem in which we add an entropic penalty function to the objective in Eq.~\eqref{prob:MOT}. The resulting problem is in the following form:
\begin{eqnarray}\label{prob:MOT_regularized}
& \min \limits_{X \in \br^{n \times \cdots \times n}} & \left\langle C, X\right\rangle - \eta H(X) \\
& \st & r_k(X) = r_k \textnormal{ for any } k \in [m], \ X \geq 0, \nonumber
\end{eqnarray}
where $\eta > 0$ denotes the regularization parameter and $H(X)$ denotes the entropic regularization term, which is given by:
\begin{equation*}
H(X) \mydefn - \langle X, \log(X)-\mathbf{1}_{n \times \cdots \times n}\rangle.
\end{equation*}
It is important to note that if $\eta$ is large, the resulting optimal value of the entropic regularized MOT problem (cf. Eq~\eqref{prob:MOT_regularized}) yields a poor approximation to the unregularized MOT problem. Moreover, another issue of entropic regularization is that the sparsity of the solution is lost. Even though an $\varepsilon$-approximate transportation plan can be found efficiently, it is not clear how different the resulting sparsity pattern of the obtained solution is with respect to the solution of the actual OT problem. In contrast, as a special instance of the MOT distance, the actual OT distance suffers from the curse of dimensionality~\citep{Dudley-1969-Speed, Fournier-2015-Rate, Weed-2019-Sharp, Lei-2020-Convergence} and is significantly worse than its entropic regularized version in terms of the sample complexity~\citep{Genevay-2019-Sample, Mena-2019-Statistical}. This statistical drawback also holds true for the unregularized MOT distance in general.

While there is an ongoing debate in the literature on the merits of solving the actual OT problem~\textit{versus.} its entropic regularized version, we adopt here the viewpoint that reaching an additive approximation of the actual MOT cost matters and therefore propose to scale $\eta$ as a function of the desired accuracy of the approximation.

Then we proceed to derive the dual form of the entropic regularized MOT problem in Eq.~\eqref{prob:MOT_regularized}. As in the usual 2-marginals OT case~\citep{Cuturi-2018-Semidual}, the dual form of the MOT problem with $m \geq 3$ remains an unconstrained smooth optimization problem. 

By introducing the dual variables $\{\lambda_1, \ldots, \lambda_m\} \subseteq \br^n$ and $\tau \in \br$, we can define the Lagrangian function of the entropic regularized MOT problem in Eq.~\eqref{prob:MOT_regularized} as follows:
\begin{equation}\label{opt:Lagrangian}
\LCal(X, \lambda_1, \ldots, \lambda_m) = \langle C, X\rangle - \eta H(X) - \sum_{k=1}^m \lambda_k^\top(r_k(X) - r_k). 
\end{equation}
Note that the entropy function $H(X)$ is not well defined for any negative matrix $X$. Thus, we can neglect the non-negative constraint $X \geq 0$ and define the above function $\LCal$ whose domain is $\br^{n \times \ldots \times n}_+ \times \br^{nm}$. In order to derive the smooth dual objective function, we consider the following minimization problem:
\begin{equation*}
\min_{X: \|X\|_1=1} \langle C, X\rangle - \eta H(X) - \sum_{k=1}^m \lambda_k^\top(r_k(X) - r_k). 
\end{equation*}
In the above problem, the objective function is strongly convex. Thus, the optimal solution is unique. After the simple calculations, the optimal solution $\bar{X} = X(\lambda_1, \ldots, \lambda_m)$ has the following form:
\begin{equation}\label{opt:MOT_plan}
\bar{X}_{i_1 \ldots i_m} = \frac{e^{\eta^{-1}(\sum_{k=1}^m \lambda_{ki_k} - C_{i_1 i_2 \ldots i_m})}}{\sum_{1 \leq i_k \leq n, \forall k \in [m]}e^{\eta^{-1}(\sum_{k=1}^m \lambda_{ki_k} - C_{i_1 i_2 \ldots i_m})}}. 
\end{equation}
Plugging Eq.~\eqref{opt:MOT_plan} into Eq.~\eqref{opt:Lagrangian} yields that the dual form is: 
\begin{equation*}
\max_{\lambda_1, \ldots, \lambda_m} \ \left\{- \eta\log\left(\sum_{1 \leq i_1, \ldots, i_m \leq n} e^{\eta^{-1}(\sum_{k=1}^m \lambda_{ki_k} - C_{i_1 i_2 \ldots i_m})}\right) + \sum_{k=1}^m \lambda_k^\top r_k\right\}.
\end{equation*}
In order to streamline our subsequent presentation, we perform a change of variables,  $\beta_k = \eta^{-1}\lambda_k$, and reformulate the above problem as
\begin{equation*}
\min_{\beta_1, \ldots, \beta_m} \varphi(\beta_1, \ldots, \beta_m) \mydefn \log\left(\sum_{1 \leq i_1, i_2, \ldots, i_m \leq n} e^{\sum_{k=1}^m \beta_{ki_k}-\frac{C_{i_1 i_2 \ldots i_m}}{\eta}}\right) - \sum_{k=1}^m \beta_k^\top r_k.  
\end{equation*}
To further simplify the notation, we define $B(\beta) : = (B_{i_1 \ldots i_m})_{i_1, i_2, \ldots, i_m \in [n]} \in \br^{n \times \ldots \times  n}$ where $\beta = (\beta_{1}, \ldots, \beta_{m})$ by
\begin{equation*}
B_{i_1 \ldots i_m} = e^{\sum_{k=1}^m \beta_{ki_k}-\frac{C_{i_1 i_2 \ldots i_m}}{\eta}}. 
\end{equation*}
To this end, we obtain the \emph{dual entropic regularized MOT problem} defined by
\begin{equation}\label{prob:MOT_regularized_dual}
\min \limits_{\beta_1, \ldots, \beta_m} \varphi(\beta_1, \ldots, \beta_m) \mydefn \log(\|B(\beta_1, \ldots, \beta_m)\|_1) - \sum_{k=1}^m \beta_k^\top r_k.  
\end{equation}
\begin{remark}\label{remark:MOT_regularized_dual}
The first part of the objective function $\varphi$ is in the form of the logarithm of sum of exponents while the second part is a linear function. This is different from the objective function used in previous dual entropic regularized OT problem~\citep{Cuturi-2013-Sinkhorn, Altschuler-2017-Near, Dvurechensky-2018-Computational, Lin-2019-Efficient}. We also note that Eq.~\eqref{prob:MOT_regularized_dual} is a special instance of a softmax minimization problem, and the objective function $\varphi$ is known to be smooth~\citep{Nesterov-2005-Smooth}. Finally, we point out that the same problem was derived in the later work by~\citet{Tupitsa-2020-Multimarginal} and used for analyzing the accelerated alternating minimization algorithm. 
\end{remark}
In the remainder of the paper, we also denote $\beta^\star = (\beta_1^\star, \ldots, \beta_m^\star) \in \Rspace^{mn}$ as an optimal solution of the dual entropic regularized MOT problem in Eq.~\eqref{prob:MOT_regularized_dual}.
%%%%%%%%%%%%%%%%%%%%%%%%%%%%%%%%%%%%%%%%%%%%%%%%%%%%%%%%%%%%%%%%%%%%%%%%%%%%%
\subsection{Properties of dual entropic regularized multimarginal OT}
In this section, we present several useful properties of the dual entropic regularized MOT in Eq.~\eqref{prob:MOT_regularized_dual}. In particular, we show that there exists an optimal solution $\beta^\star$ such that it has an upper bound in terms of the $\ell_\infty$-norm. 
\begin{lemma}\label{Lemma:dual-bound-infinity}
For the dual entropic regularized MOT problem in Eq.~\eqref{prob:MOT_regularized_dual}, there exists an optimal solution $\beta^\star = (\beta_1^\star, \ldots, \beta_m^\star)$ such that 
\begin{equation}\label{lemma-dual-bound-main}
\|\beta^\star\|_\infty \mydefn \max_{1 \leq i \leq m} \|\beta_i^*\|_\infty \leq R, 
\end{equation} 
where $R > 0$ is defined as
\begin{equation*}
R \mydefn \frac{\|C\|_\infty}{\eta} - \log\left(\min_{1 \leq i \leq m, 1 \leq j \leq n} r_{ij}\right).
\end{equation*}
\end{lemma}
\begin{proof}
First, we claim that there exists an optimal solution $\beta^\star =(\beta_1^\star, \ldots, \beta_m^\star)$ such that
\begin{equation}\label{claim-dual-bound-first}
\min\limits_{1 \leq j \leq n} \beta_{ij}^\star \leq 0 \leq \max\limits_{1 \leq j \leq n} \beta_{ij}^\star \textnormal{ for any } i \in [m]. 
\end{equation}
Indeed, letting $\widehat{\beta}^\star=(\widehat{\beta}_1^\star, \ldots, \widehat{\beta}_m^\star)$ be an optimal solution to Eq.~\eqref{prob:MOT_regularized_dual}, the claim holds true if $\widehat{\beta}^\star$ satisfies Eq.~\eqref{claim-dual-bound-first}. Otherwise, we let $m$ shift terms be
\begin{equation*}
\Delta\widehat{\beta}_i \ = \ \frac{\max_{1 \leq j \leq n} \widehat{\beta}_{ij}^\star + \min_{1 \leq j \leq n} \widehat{\beta}_{ij}^\star}{2} \in \Rspace \textnormal{ for any } i \in [m].  
\end{equation*}
and define $\beta^\star=(\beta_1^\star, \ldots, \beta_m^\star)$ by  
\begin{equation*}
\beta_i^\star \ = \ \widehat{\beta}_i^\star - \Delta\widehat{\beta}_i\one_n \textnormal{ for any } i \in [m]. 
\end{equation*}
By the definition of $\beta^\star$, it is clear that $\beta^\star$ satisfies Eq.~\eqref{claim-dual-bound-first}. Since $\one_n^\top r_i = 1$ for all $i \in [m]$, we have $(\beta_i^\star)^\top r_i = (\widehat{\beta}_i^\star)^\top r_i - \Delta\widehat{\beta}_i$ for all $i \in [m]$. In addition, we have $\log(\|B(\beta_1^\star, \ldots, \beta_m^\star)\|_1) = \log(\|B(\widehat{\beta}_1^\star, \ldots, \widehat{\beta}_m^\star)\|_1) + \sum_{i=1}^m \Delta\widehat{\beta}_i$. Putting these pieces together yields $\varphi(\beta^\star) = \varphi(\widehat{\beta}^\star)$. Therefore, $\beta^\star$ is an optimal solution that satisfies Eq.~\eqref{claim-dual-bound-first}.

Then, we show that  
\begin{equation}\label{claim-dual-bound-second}
\max\limits_{1 \leq j \leq n} \beta_{ij}^\star - \min \limits_{1 \leq j \leq n} \beta_{ij}^\star \leq \frac{\|C\|_{\infty}}{\eta} - \log\left(\min_{1 \leq i \leq m, 1 \leq j \leq n} r_{ij}\right) \textnormal{ for all } i \in [m].  
\end{equation}
Indeed, for any $(j, l) \in [m] \times [n]$, we derive from the optimality condition of $\beta^\star$ that
\begin{equation*}
\frac{e^{\beta_{jl}^\star}\sum_{1 \leq i_k \leq n, \forall k \neq j} e^{\sum_{k \neq j} \beta_{ki_k}^\star - \eta^{-1}C_{i_1 \cdots l \cdots i_m}}}{\|B(\beta_1^\star, \ldots, \beta_m^\star)\|_1} \ = \ r_{jl} \ \geq \ \min_{1 \leq i \leq m, 1 \leq j \leq n} r_{ij}.   
\end{equation*}
Since $C$ is a nonnegative cost tensor, we have
\begin{equation}\label{inequality-bound-first}
\beta_{jl}^\star \geq \log\left(\min_{1 \leq i \leq m, 1 \leq j \leq n} r_{ij}\right) - \log \left(\sum\limits_{1 \leq i_k \leq n, \forall k \neq j} e^{\sum_{k \neq j} \beta_{ki_k}^\star} \right) + \log(\|B(\beta_1^\star, \ldots, \beta_m^\star)\|_1).
\end{equation}
Since $r_{jl} \in [0, 1]$ and $C_{i_1 \ldots i_m} \leq \|C\|_\infty$, we have
\begin{equation}\label{inequality-bound-second}
\beta_{jl}^\star \leq \frac{\|C\|_\infty}{\eta} - \log \left( \sum \limits_{1 \leq i_k \leq n, \forall k \neq j} e^{\sum_{k \neq j} \beta_{ki_k}^\star} \right) + \log(\|B(\beta_1^\star, \ldots, \beta_m^\star)\|_1).
\end{equation}
Combining the bounds in Eq.~\eqref{inequality-bound-first} and Eq.~\eqref{inequality-bound-second} implies the desired Eq.~\eqref{claim-dual-bound-second}. 

Finally, we prove that Eq.~\eqref{lemma-dual-bound-main} holds true. Indeed, Eq.~\eqref{claim-dual-bound-first} and Eq.~\eqref{claim-dual-bound-second} imply that  
\begin{equation}\label{inequality-bound-third}
-\frac{\|C\|_{\infty}}{\eta} + \log\left(\min_{1 \leq i \leq m, 1 \leq j \leq n} r_{ij}\right) \leq \min\limits_{1 \leq j \leq n} \beta_{ij}^\star \leq 0 \text{ for any } i \in [m],  
\end{equation}  
and 
\begin{equation}\label{inequality-bound-fourth}
0 \leq \max\limits_{1 \leq j \leq n} \beta_{ij}^\star \leq \frac{\|C\|_\infty}{\eta} - \log\left(\min_{1 \leq i \leq m, 1 \leq j \leq n} r_{ij}\right) \text{ for any } i \in [m]. 
\end{equation}
Combining Eq.~\eqref{inequality-bound-third} and Eq.~\eqref{inequality-bound-fourth} with the definition of $R$ implies that $\max_{1 \leq i \leq m} \|\beta_i^\star\|_\infty \leq R$ and hence the desired Eq.~\eqref{lemma-dual-bound-main}. 
\end{proof}
The upper bound for the $\ell_{\infty}$-norm of an optimal solution of dual entropic-regularized multimarginal OT in Lemma~\ref{Lemma:dual-bound-infinity} directly leads to the following direct bound for the $\ell_{2}$-norm. 
\begin{corollary}\label{Corollary:dual-bound-l2}
For the dual entropic regularized MOT problem in Eq.~\eqref{prob:MOT_regularized_dual}, there exists an optimal solution $\beta^\star = (\beta_1^\star, \ldots, \beta_m^\star)$ such that 
\begin{equation*}
\|\beta^*\| \leq \sqrt{mn}R, 
\end{equation*} 
where $R > 0$ is defined in Lemma~\ref{Lemma:dual-bound-infinity}. 
\end{corollary}
Since the function $-H(X)$ is strongly convex with respect to the $\ell_1$-norm on the probability simplex $Q \subseteq \Rspace^{n^m}$, the entropic regularized MOT problem in Eq.~\eqref{prob:MOT_regularized} is a special case of the following linearly constrained convex optimization problem: 
\begin{equation*}
\min_{x \in Q} \ f(x), \quad \st \ Ax = b, 
\end{equation*}
where $f$ is strongly convex with respect to the $\ell_1$-norm on the set $Q$: 
\begin{equation*}
f(x') - f(x) - (x' - x)^\top\nabla f(x) \geq \frac{\eta}{2}\|x' - x\|_1^2 \textnormal{ for any } x', x \in Q.  
\end{equation*}
We use the $\ell_2$-norm for the dual space of the Lagrange multipliers. By~\citet[Theorem~1]{Nesterov-2005-Smooth}, the dual objective function $\tilde{\varphi}$ satisfies the following inequality:  
\begin{equation*}
\widetilde{\varphi}(\lambda') - \widetilde{\varphi}(\lambda) - (\lambda' - \lambda)^\top\nabla\widetilde{\varphi}(\lambda) \leq \frac{\|A\|_{1 \rightarrow 2}^2}{2\eta}\|\lambda' - \lambda\|^2 \textnormal{ for any } \lambda', \lambda \in \Rspace^{mn}.   
\end{equation*}
Recall that the function $\tilde{\varphi}$ is given by 
\begin{equation*}
\widetilde{\varphi}(\lambda) = - \eta\log\left(\sum_{1 \leq i_1, \ldots, i_m \leq n} e^{\eta^{-1}(\sum_{k=1}^m \lambda_{ki_k} - C_{i_1 i_2 \ldots i_m})}\right) + \sum_{k=1}^m \lambda_k^\top r_k.
\end{equation*}
We notice that the function $\varphi$ in Eq.~\eqref{prob:MOT_regularized_dual} is defined by 
\begin{equation*}
\varphi(\beta) = - \eta^{-1}\widetilde{\varphi}(\eta(\beta + (1/m)\one_{mn})). 
\end{equation*}
After some calculations, we have
\begin{equation}\label{inequality-gradient-objective}
\varphi(\beta') - \varphi(\beta) - (\beta' - \beta)^\top\nabla\varphi(\beta) \leq \left(\frac{\|A\|_{1 \rightarrow 2}^2}{2}\right)\|\beta' - \beta\|^2. 
\end{equation}
By definition, each column of the matrix $A$ contains no more than $m$ nonzero elements which are equal to one. Since $\|A\|_{1 \rightarrow 2}$ is equal to maximum $\ell_2$-norm of the column of this matrix, we have $\|A\|_{1 \rightarrow 2} = \sqrt{m}$. Thus, the dual objective function $\varphi$ is $m$-gradient Lipschitz with respect to the $\ell_2$-norm. This implies that the squared norm of the gradient is bounded by the dual objective gap~\citep{Nesterov-2018-Lectures}. We present this result in the following lemma and provide the proof for the sake of completeness. 
\begin{lemma}\label{Lemma:gradient-objective-l2}
For any given vector $\beta \in \Rspace^{n m}$, we have
\begin{equation*}
\sum_{i=1}^m \left(\varphi(\beta) - \argmin_{\gamma \in \Rspace^n}\varphi(\beta_1, \ldots, \beta_{i-1}, \gamma, \beta_{i+1}, \ldots, \beta_m)\right) \ \geq \ \left(\frac{1}{2m}\right)\|\nabla \varphi(\beta)\|^2. 
\end{equation*} 
\end{lemma}
\begin{proof}
We derive from Eq.~\eqref{inequality-gradient-objective} with $\bar{\beta}_i = \beta_i - \frac{1}{m}\nabla_{\beta_i}\varphi(\beta)$ and $\bar{\beta}_k = \beta_k$ for $k \neq i$ that 
\begin{equation*}
\varphi(\beta) - \argmin_{\gamma \in \Rspace^n}\varphi(\beta_1, \ldots, \beta_{i-1}, \gamma, \beta_{i+1}, \ldots, \beta_m) \leq \varphi(\beta) - \varphi(\bar{\beta}) \leq \varphi(\beta) - \left(\frac{1}{2m}\right)\|\nabla_i\varphi(\beta)\|^2. 
\end{equation*}
Summing up the above inequality over $i \in [m]$ yields the desired inequality. 
\end{proof}

%!TEX root = paper.tex
\section{Computational Hardness}\label{sec:hardness}
In this section, we show that the multimarginal optimal transport (MOT) problem in the form of Eq.~\eqref{prob:MOT} is not a minimum-cost flow problem when $m \geq 3$. The proof idea is based on a simple reduction with $m$-dimensional matching problem.

\subsection{Unimodularity, minimum-cost flow and matching}
We present some definitions and classical results in combinatorial optimization and graph theory, including unimodularity, minimum-cost flow and matching. 
\begin{definition}\label{Def:TUM}
A totally unimodular (TU) matrix is one for which every square submatrix has determinant $-1$, $0$ or $1$.
\end{definition}
A direct way to determine whether a matrix is totally unimodular or not is by computing the determinants of every square submatrix of this matrix. However, it is clearly intractable in general. The following proposition provides an alternative way to check whether a matrix is TU or not. 
\begin{proposition}\label{Prop:TUM}
Let $A$ be a $\{-1, 0, 1\}$-valued matrix. $A$ is TU if each column contains at most two nonzero entries and all rows are partitioned into two sets $I_1$ and $I_2$ such that: If two nonzero entries of a column have the same sign, they are in different sets. If these two entries have different signs, they are in the same set.
\end{proposition}
In what follows, we present the definition of minimum-cost flow problem and prove that the constraint matrix of LP representation of a minimum-cost flow problem is TU. Such result is well known and can be derived from~\citet[Theorem~1, Chapter 15]{Berge-2001-Theory} which shows that the incidence matrices of every directed graphs are TU. For the sake of completeness, we provide the detailed proof based on Proposition~\ref{Prop:TUM}. 
\begin{definition}\label{Def:MCF}
The minimum-cost flow problem finds the cheapest possible way of sending a certain amount of flow through a flow network. Formally, 
\begin{equation*}
\begin{array}{ll}
\min & \sum_{(u,v) \in E} f(u,v) \cdot a(u,v) \\
\st & f(u,v) \geq 0, \; \textnormal{ for all} \; (u,v) \in E, \\
& f(u,v) \leq c(u,v) \; \textnormal{ for all} \; (u,v) \in E, \\
& f(u,v) = - f(v,u) \;  \textnormal{ for all} \;   (u,v) \in E, \\
& \sum_{(u,w)\in E \;  \textnormal{or} \; (w, u) \in E } f(u,w) = 0, \\
& \sum_{w \in V} f(s,w) = d \;  \textnormal{ and } \;  \sum_{w \in V} f(w,t) = d.
\end{array}
\end{equation*}
The flow network $G = (V, E)$ is a directed graph $G = (V, E)$ with a source vertex $s \in V$ and a sink vertex $t \in V$, where each edge $(u, v) \in E$ has capacity $c(u,v) > 0$, flow $f(u,v) \geq 0$ and cost $a(u, v)$, with most minimum-cost flow algorithms supporting edges with negative costs. The cost of sending this flow along an edge $(u, v)$ is $f(u,v) \cdot a(u,v)$. The problem requires an amount of flow $d$ to be sent from source $s$ to sink $t$. The definition of the problem is to minimize the total cost of the flow over all edges. 
\end{definition}
\begin{proposition}\label{Prop:MCF}
The constraint matrix arising from a minimum-cost flow problem is TU.
\end{proposition}
\begin{proof}
The standard LP representation of the minimum-cost flow problem is 
\begin{equation*}
\min_{x \in \br^{|E|}} \ c^\top x, \quad \st \ Ax = b, \ l \leq x \leq u. 
\end{equation*}
where $x \in \br^{|E|}$ with $x_j$ being the flow through arc $j$, $b \in \br^{|V|}$ with $b_i$ being external supply at node $i$ and $\one^\top b = 0$, $c_j$ is unit cost of flow through arc $j$, $l_j$ and $u_j$ are lower and upper bounds on flow through arc $j$ and $A \in \br^{|V| \times |E|}$ is the arc-node incidence matrix with entries 
\begin{equation*}
A_{ij} \ = \ \left\{\begin{array}{rl} -1 & \text{if arc $j$ starts at node $i$} \\ 1 & \text{if arc $j$ ends at node $i$} \\ 0 & \text{otherwise} \end{array}\right..
\end{equation*}
Since each arc has two endpoints, the constraint matrix $A$ is a $\{-1, 0, 1\}$-valued matrix in which each column contains two nonzero entries $1$ and $-1$. Using Proposition~\ref{Prop:TUM}, we obtain that $A$ is TU and the rows of $A$ are categorized into a single set. 
\end{proof}
We proceed to the definition of $m$-dimensional matching which generalizes 2-dimensional matching. We present it in graph-theoretic sense as follows. 
\begin{definition}
Let $S_1, S_2, \ldots, S_m$ be finite and disjoint sets, and let $T$ be a subset of $S_1 \times \cdots \times S_m$. That is, $T$ consists of vectors $(z_1, \ldots, z_m)$ such that $z_i \in S_i$ for all $i \in [m]$. Now $M \subseteq T$ is a $m$-dimensional matching if the following holds: for any two distinct vectors $(z_1, \ldots, z_m) \in M$ and $(z'_1, \ldots, z'_m) \in M$, we have $z_i \neq z'_i$ for all $i \in [m]$.  
\end{definition}
In computational complexity theory, $m$-dimensional matching refers to the following decision problem: given a set $T$ and an integer $k$, decide whether there exists a $m$-dimensional matching $M \subseteq T$ with $|M| \geq k$. This problem is NP-complete even when $m = 3$ and $k = |S_1| = |S_2| = |S_3|$~\citep{Karp-1972-Reducibility, Garey-2002-Computers}. A $m$-dimensional matching is also an exact cover since the set $M$ covers each element of $S_1, S_2, \ldots, S_m$ exactly once. 

\subsection{Main result}\label{sec:problem_setup}
The problem of computing the MOT distance between $m \geq 2$ discrete probability distributions with at most $n$ supports is equivalent to solving the following LP (cf. Eq.~\eqref{prob:MOT}):
\begin{equation*}
\min_{X \in \br^{n \times \cdots \times n}} \langle C, X\rangle, \quad \st \ r_k(X) = r_k \textnormal{ for any } k \in [m], \ X \geq 0. 
\end{equation*}
In other words, the MOT problem is a LP with $mn$ equality constraints and $n^m$ variables. When $m = 2$, the MOT problem is the classical OT problem~\citep{Villani-2003-Topic} which is known to be a minimum-cost flow problem. Such problem structure is computationally favorable and permits the development of provably efficient algorithms, including the network simplex algorithms~\citep{Orlin-1997-Polynomial, Tarjan-1997-Dynamic} and specialized interior-point algorithms~\citep{Lee-2014-Path}. However, it remains unknown if the MOT problem in the above LP form admits such a structural decomposition when $m \geq 3$. 

We present a negative answer to this question for $m \geq 3$. Before proceeding to the main theorem, we provide a simple yet intuitive counterexample. 
\begin{example}\label{example:MOT}
We consider arguably the simplest MOT problem, with $m=3$ distributions supported on $n=2$ elements each. We consider the $n^m=8$ entries of a multimarginal tensor transportation plan, and number them slice by slice. A naive enumeration of all the marginal constraints results in $nm$ linear equalities, but some of them are redundant since they involve several times the constraints that the sum of the elements of that tensor sum to $1$. The number of required constraints is $m(n-1)+1$, namely only $4$ mass conservation constraints are effective in this case. We therefore obtain the following matrix,
\begin{align*}
A  = \begin{pmatrix}
1 & 1 & 0 & 0 & 1 & 1 & 0 & 0 \\ 0 & 0 & 1 & 1 & 0 & 0 & 1 & 1 \\ 1 & 0 & 1 & 0 & 1 & 0 & 1 & 0 \\ 1 & 1 & 1 & 1 & 0 & 0 & 0 & 0
\end{pmatrix}.
\end{align*}
We form the sub-matrix by only considering the first, fourth, sixth, and seventh columns of $A$, and can then check that the resulting matrix has determinant equal to 2, namely,
\begin{align*}
\det(A_{1,4,6,7}) = \det\left(\begin{pmatrix}
1 & 0 & 1 & 0 \\ 
0 & 1 & 0 & 1 \\ 
1 & 0 & 0 & 1 \\ 
1 & 1 & 0 & 0 
\end{pmatrix}\right) = 2.
\end{align*}
Therefore, the marginal constraint matrix is not totally unimodular, illustrating that the MOT with $(m, n) = (3, 2)$ is not a minimum-cost flow problem. More generally, one can numerically check that the constraint matrix corresponding to $m$ marginals with $n$ points each has size $(mn-m+1)\times n^m$, and that it is not totally unimodular by selecting a subset of $(mn-m+1)$ columns (out of $n^m$) that form a determinant that is neither $-1,0,1$. The constraint matrix itself can be obtained recursively, by defining first $L_{n, 1} = I_n$, to apply next that for $t\geq 2$,
\begin{equation*}
L_{n, t} = \begin{bmatrix} 
\one_n \otimes L_{n,t-1} & I_n \otimes \one_{n^{t-1}} 
\end{bmatrix} \ \in \ \Rspace^{n^t\times nt},
\end{equation*}
where $\otimes$ is Kronecker's product. In that case, $L_{n, m}$ corresponds to the matrix constraint of the dual multimarginal OT problem, which involves constraints of the type $(\alpha_1)_{i_1} + (\alpha_2)_{i_1} + \dots + (\alpha_m)_{i_m} \leq C_{i_1 i_2 \ldots i_m}$ as mentioned in the next section. The constraint matrix in the primal, specified over the entries of transportation tensors, is $A_{n, m} = \tilde{L}_{n,m}^{\top}$, where $\tilde{L}_{n, m}$ is equal to $L_{n, m}$ stripped of $m-1$ columns (one for each marginal but for the first), indexed for instance at $2n, 3n, \ldots, nm$. 
\end{example}
Example~\ref{example:MOT} provides some intuitions why the MOT problem is not a minimum-cost flow problem when $m \geq 3$. However, it is not easy to extend this approach to the general setting. Indeed, the constraint matrix in Eq.~\eqref{prob:MOT} becomes complicated when $m$ and $n$ are considerably large. Thus, it is challenging to compute the determinants of even a small fraction of sub-matrices, which is necessary to determine whether the constraint matrix is totally unimodular or not. While the direct calculation is intractable, some combinatorial optimization toolbox, e.g., Ghouila-Houri’s theorem~\citep{Ghouila-1962-Caracterisation}, might be helpful. However, we do not have concrete idea now and leave this topic to the future work. 

Despite the above discussion, we can prove that the MOT problem in Eq.~\eqref{prob:MOT} is not a minimum-cost flow problem when $m \geq 3$ by using a simple reduction with $m$-dimensional matching problem. Roughly speaking, if the MOT problem is a minimum-cost flow problem when $m \geq 3$, its integer programming counterpart with specific choice of the cost tensor $C$ and marginals $\{r_k\}$ must not be NP-hard. However, due to such specific choice, we can prove that this integer programming counterpart is equivalent to $m$-dimensional matching problem which is known as NP-complete when $m \geq 3$. This leads to the contradiction.

We present our theorem with the proof details as follows. 
\begin{theorem}\label{Theorem:MOT-structure}
The MOT problem in the form of Eq.~\eqref{prob:MOT} is not a minimum-cost flow problem when $m \geq 3$. 
\end{theorem}
\begin{proof}
We prove the result by contradiction. Indeed, we assume that the MOT problem in Eq.~\eqref{prob:MOT} is a minimum-cost flow problem when $m \geq 3$. Let $r_k = \frac{\one_n}{n}$ for all $k \in [m]$ in Eq.~\eqref{prob:MOT}, the resulting LP is equivalent to the following problem 
\begin{equation}\label{prob:LP_MOT}
\min\limits_X \ \langle C, X\rangle, \quad \st \ r_k(X) = \one_n\textnormal{ for any } k \in [m], \ \ X \geq 0. 
\end{equation}
We see from Eq.~\eqref{prob:LP_MOT} that this is a minimum-cost flow problem where the constraint matrix and the right-hand side vector are both integer-valued. Then we consider the integer programming counterpart of Eq.~\eqref{prob:LP_MOT} which is defined by
\begin{eqnarray}\label{prob:IP_MOT}
& \min\limits_X & \langle C, X\rangle, \\ 
& \st & r_k(X) = \one_n, \textnormal{ for any } k \in [m], \nonumber \\ 
& & X_{i_1 i_2 \ldots i_m} \in \{0, 1\} \textnormal{ for any } (i_1, \ldots, i_m) \in [n] \times \ldots \times [n]. \nonumber
\end{eqnarray}
It is well known in the combinatorial optimization literature~\citep{Schrijver-2003-Combinatorial} that Eq.~\eqref{prob:IP_MOT} is not NP-hard when $m \geq 3$.  

On the other hand, we claim that Eq.~\eqref{prob:IP_MOT} is NP-complete when $m \geq 3$ since it reduces to an $m$-dimensional matching problem. Indeed, we let $S_i = [n]$ for all $i \in [m]$ and $T = [n] \times [n] \times \ldots \times [n]$ as well as the cost tensor $C$ is defined by 
\begin{equation*}
C_{i_1 i_2 \ldots i_m} = \left\{\begin{array}{ll} 1 & (i_1, i_2, \ldots, i_m) \in T \\ 0 & (i_1, i_2, \ldots, i_m) \notin T \end{array}\right..
\end{equation*}
Then the objective function of any feasible solution is $n$ so any feasible solution is an optimal solution. Furthermore, finding any optimal solution $X$ is equivalent to finding an $m$-dimensional matching $M$. Indeed, we can define an one-to-one mapping as follows,   
\begin{equation*}
X_{i_1 i_2 \ldots i_m} = \left\{\begin{array}{ll} 1 & (i_1, i_2, \ldots, i_m) \in M \\ 0 & (i_1, i_2, \ldots, i_m) \notin M \end{array}\right..
\end{equation*}
It is clear that $X$ is a feasible solution of Eq.~\eqref{prob:IP_MOT} if and only if $M$ is an $m$-dimensional matching $M$. Thus, Eq.~\eqref{prob:IP_MOT} is NP-complete when $m \geq 3$, which leads to a clear contradiction. This completes the proof of Theorem~\ref{Theorem:MOT-structure}. 
\end{proof}

\subsection{Discussion} 
We make a few comments on our main result for the MOT problem in Eq.~\eqref{prob:MOT}, which help strengthen the understanding of the MOT problem. 

First, Theorem~\ref{Theorem:MOT-structure} only holds true for the general MOT problem in Eq.~\eqref{prob:MOT}. To be more specific, we show that there exists the cost tensor $C$ and marginals $\{r_k\}$ such that the MOT problem in Eq.~\eqref{prob:MOT} is not a minimum-cost flow problem. Nonetheless, the MOT problem is commonly referred to as the LP in Eq.~\eqref{prob:MOT} and thus it is important to understand the structure of this LP when $m \geq 3$.  Further, we remark that there is no other reformulation of the general MOT problems which can be solved efficiently; in particular, ~\citet{Altschuler-2021-Hardness} showed that the MOT problems with repulsive costs are computationally intractable: several such problems of interest are NP-hard to solve -- even approximately. 

Second, Theorem~\ref{Theorem:MOT-structure} does not rule out the possibility that a few instances of the MOT problem are minimum-cost flow problems when $m \geq 3$. However, such examples need to admit special structure and are thus rare in real applications given that Example~\ref{example:MOT} is one of the simplest MOT problems. Many common MOT problems, e.g., Wasserstein barycenters, are not minimum-cost flow problems. Indeed,~\citet{Lin-2020-Revisiting} proved this result for even the simplest Wasserstein barycenter problems in the standard LP form --- the fixed-support Wasserstein barycenters when $m \geq 3$ and $n \geq 3$. Despite such negative result, the discrete Wasserstein barycenter problems have their own structure~\citep{Anderes-2016-Discrete} and can be efficiently solved in practice~\citep{Cuturi-2014-Fast, Benamou-2015-Iterative, Carlier-2015-Numerical, Staib-2017-Parallel, Claici-2018-Stochastic, Dvurechensky-2018-Decentralize, Kroshnin-2019-Complexity, Ge-2019-Interior}. In conclusion, the MOT problem with $m \geq 3$ is different from the OT problem and we believe that the minimum-cost flow is not sufficient for characterizing the structure of the MOT problem. 

Finally, Theorem~\ref{Theorem:MOT-structure} affects the complexity bound of various algorithms for solving the MOT problem when $m \geq 3$. If we could write the MOT problem as a minimum-cost flow problem on the directed graph with $n^m$ edges and $mn$ vertices, the network simplex method achieves the complexity bound of $\bigOtil(m^2n^{m+2})$~\citep{Orlin-1997-Polynomial} or better bound of $\bigOtil(mn^{m+1})$~\citep{Tarjan-1997-Dynamic}, while the specialized interior-point algorithm achieves the complexity bound of $\bigOtil(\sqrt{mn}n^m)$~\citep{Lee-2014-Path}. However, due to Theorem~\ref{Theorem:MOT-structure}, these bounds of network simplex method and specialized interior-point algorithms are not valid and only the standard interior-point algorithms can achieve much worse complexity bound of $\bigOtil(n^{3m})$~\citep{Wright-1997-Primal}. We are also aware of a \textit{stochastic central path method}~\citep{Cohen-2019-Solving} which achieves better complexity bound of $\bigOtil(n^{\omega m})$ with the coefficient of matrix multiplication $\omega \approx 2.38$. However, this algorithm is seemingly not implementable in practice and not comparable with the deterministic algorithms we propose in this paper. 

%!TEX root = paper.tex
\section{Multimarginal Sinkhorn Algorithm}\label{sec:sinkhorn}
In this section, we propose and analyze a multimarginal Sinkhorn algorithm for solving the entropic regularized multimarginal optimal transport (MOT) problem. We also generalizes the rounding scheme~\citep{Altschuler-2017-Near} to the MOT setting. Together with a new rounding scheme, our algorithm achieves a complexity bound of $\bigOtil(m^3n^m\varepsilon^{-2})$ when applied to solve the MOT problem. The proof techniques are heavily based on the smooth dual objective function in Eq.~\eqref{prob:MOT_regularized_dual} and thus not a straightforward generalization of the analysis in the OT setting~\citep{Altschuler-2017-Near, Dvurechensky-2018-Computational, Lin-2019-Efficient} where they use different form of dual objective function; see Remark~\ref{remark:MOT_regularized_dual} for the details. 
\subsection{Algorithmic procedure}
We present the pseudocode of the multimarginal Sinkhorn algorithm in Algorithm~\ref{Algorithm:Sinkhorn}. This algorithm is a new generalization of the classical Sinkhorn algorithm~\citep{Cuturi-2013-Sinkhorn} and different from the existing multimarginal Sinkhorn algorithms~\citep{Benamou-2015-Iterative, Benamou-2019-Generalized, Peyre-2019-Computational}. Indeed, the main difference lies in the greedy choice of the next marginal (cf. \textbf{Step 2}). This simple yet crucial modification makes our complexity bound analysis work while only the asymptotic convergence properties are proved for the existing multimarginal Sinkhorn algorithms. 
\begin{algorithm}[!t]\small
\caption{\textsc{MultiSinkhorn}$(C, \eta, \{r_k\}_{k \in [m]}, \varepsilon')$} \label{Algorithm:Sinkhorn}
\begin{algorithmic}
\STATE \textbf{Initialization:} $t = 0$ and $\beta^0 \in \Rspace^{mn}$ with $\beta^0 = \textbf{0}_{mn}$. 
\WHILE{$E_t > \varepsilon'$} 
\STATE \textbf{Step 1.} Choose the greedy coordinate $K = \argmax_{1 \leq k \leq m} \rho(r_k, r_k(B(\beta^t)))$. 
\STATE \textbf{Step 2.} Compute $\beta^{t+1} \in \br^{mn}$ by
\begin{equation*}
\beta_k^{t+1} = \left\{\begin{array}{ll}
\beta_k^t + \log(r_k) - \log(r_k(B(\beta^t))), & k = K \\ 
\beta_k^t, & \text{otherwise}
\end{array}\right..
\end{equation*}
\STATE \textbf{Step 3.} Increment by $t = t + 1$.
\ENDWHILE
\STATE \textbf{Output:} $B(\beta^t)$.  
\end{algorithmic}
\end{algorithm}
\paragraph{Comments on algorithmic scheme.} Algorithm~\ref{Algorithm:Sinkhorn} can be interpreted as a greedy block coordinate descent algorithm~\citep{Dhillon-2011-Nearest, Nutini-2015-Coordinate} for solving the dual entropic regularized MOT problem in Eq.~\eqref{prob:MOT_regularized_dual}; see~\citet[Lemma~1]{Tupitsa-2020-Multimarginal} for the justification. However, the corresponding known complexity bounds for greedy block coordinate descent algorithms can not be applied to analyze Algorithm~\ref{Algorithm:Sinkhorn}. Indeed, the per-iteration progress is quantified using the $\ell_2$-norm in the existing algorithmic scheme and convergence analysis. This will lead to the worse complexity bound than ours since it does not respect the structure of MOT problem. In contrast, Algorithm~\ref{Algorithm:Sinkhorn} employs KL divergence to quantify the per-iteration progress and link it to the $\ell_1$-norm via appeal to the Pinsker inequality~\citep{Cover-2012-Elements}. 

More specifically, an exact coordinate update for the $K$-th variable is performed at each iteration while other variables are fixed\footnote{~\citet{Meshi-2012-Convergence} showed that the greedy rule was implemented efficiently by using a max-heap structure for many structured problems.}. Here we choose $K$ by using the greedy rule as follows, 
\begin{equation*}
K = \argmax_{1 \leq k \leq m} \ \rho(r_k, r_k(B(\beta^t))),  
\end{equation*}
where $\rho: \Rspace_+^n \times \Rspace_+^n \rightarrow \Rspace_+$ is defined as
\begin{equation*}
\rho(a, b) \mydefn \one_n^\top(b - a) + \sum_{i=1}^n a_i\log\left(\frac{a_i}{b_i}\right).  
\end{equation*}
Following up the optimal transport literature~\citep{Cuturi-2013-Sinkhorn, Altschuler-2017-Near}, we set the stopping criterion as $E_t \leq \varepsilon'$ for some tolerance $\varepsilon' > 0$, where $E_t$ is defined by
\begin{equation}\label{Def:residue-sinkhorn}
E_t \mydefn \sum_{k=1}^m \|r_k(B(\beta^t)) - r_k\|_1.
\end{equation}
\paragraph{Comments on arithmetic operations per iteration.} The most expensive step is to determine which coordinate is the greedy one. While the naive way requires $\bigO(mn^m)$ arithmetic operations to compute all marginals $r_k(B(\beta^t))$, we can adopt some implementation tricks based on the observation that \textit{one of $\rho(r_k, r_k(B(\beta^t))) = 0$ after the first step.}

Without loss of generality, we assume that $m \geq 3$ and $r_1(B(\beta^t)) = r_1$. The key step is to construct a small tensor $A$ which has $n^{m-1}$ entries: $A_{i_1, \ldots, i_{m-1}} = \sum_{j=1}^n B_{j, i_1, ..., i_{m-1}}$ for any $(i_1, \ldots, i_{m-1}) \in [n] \times \cdot \times [n]$. This requires $\bigO(n^m)$ arithmetic operations. It is clear that $r_2(B(\beta^t)), \ldots, r_m(B(\beta^t))$ exactly corresponds to the marginals of $A$ and the computation only needs $\bigO(mn^{m-1})$ arithmetic operations. Putting these pieces together yields that the arithmetic operations per iteration is $\bigO(n^m)$ for the case of $m=\bigO(n)$.  
\begin{algorithm}[!t]\small
\caption{\textsc{Round}$(X, \{r_k\}_{k \in [m]})$} \label{Algorithm:round}
\begin{algorithmic}
\STATE \textbf{Initialization:} $X^{(0)} = X$.
\FOR{$k = 1$ to $m$}
\STATE Compute $z_k = \min\{\one_n, r_k/r_k(X^{(k - 1)})\} \in \Rspace^n$. 
\FOR{$j = 1$ to $n$}
\STATE $X_{i_1 i_2 \ldots i_m}^{(k)} = z_{kj} X_{i_1 i_2 \ldots i_m}^{(k - 1)}$ in which $i_k = j$ is fixed. 
\ENDFOR
\ENDFOR
\STATE Compute $\text{err}_k \in \Rspace^n$ such that $\text{err}_k = r_k - r_k(X^{(m)})$ for all $k \in [m]$. 
\STATE Compute $Y \in \Rspace^{n \times n \times \ldots \times n}$ by
\begin{equation*}
Y_{i_1 i_2 \ldots i_m} = X_{i_1 i_2 \ldots i_m}^{(m)} + \frac{\prod_{k=1}^m \text{err}_{k i_k}}{\|\text{err}_1\|_1^{m-1}} \textnormal{ for any } (i_1, \ldots, i_m) \in [n] \times \ldots \times [n]. 
\end{equation*} 
\STATE \textbf{Output:} $Y$.  
\end{algorithmic}
\end{algorithm} 
\paragraph{Rounding scheme.} Algorithm~\ref{Algorithm:Sinkhorn} is developed for solving the entropic-regularized MOT problem and the output is not necessarily a feasible solution of \textit{unregularized} MOT problem. To address this issue, we develop a new rounding scheme by extending~\citet[Algorithm~2]{Altschuler-2017-Near} to the MOT setting; see Algorithm~\ref{Algorithm:round}. We can see that the difference between the input and output of Algorithm~\ref{Algorithm:round} is simply a rank-one tensor. Using the approach presented by~\citep[Proposition~4]{Lacombe-2018-Large}, we can compute $\langle C, Y\rangle$ efficiently using $\langle C, X\rangle$. Finally, the total arithmetic operations required by Algorithm~\ref{Algorithm:round} is $\bigO(mn^{m})$.

\paragraph{Algorithm for the MOT problem.} We present the pseudocode of our main algorithm in Algorithm~\ref{Algorithm:ApproxMOT_Sinkhorn}, where Algorithm~\ref{Algorithm:Sinkhorn} and~\ref{Algorithm:round} are the subroutines. We notice that the regularization parameter $\eta$ is scaled as a function of the desired accuracy $\varepsilon > 0$, and remark that \textbf{Step 1} is necessary since the multimarginal Sinkhorn algorithm is not well behaved if the marginal distributions do not have dense support.
%%%%%%%%%%%%%%%%%%%%%%%%%%%%%%%%%%%%%%%%%%%%%%%%%%%%%%%%%%%%%%%%%%%%
\subsection{Technical lemmas}
In this section, we provide two technical lemmas which are important in the analysis of Algorithm~\ref{Algorithm:Sinkhorn}. The first lemma shows that the dual objective gap at iteration $t$ is bounded by the product between the residue term $E_t$ and a constant depending on $\eta$, $C$ and $\{r_i\}_{i \in [m]}$. 
\begin{lemma}\label{Lemma:sinkhorn-objgap}
Let $\{\beta^t\}_{t \geq 0}$ be the iterates generated by Algorithm~\ref{Algorithm:Sinkhorn} and $\beta^\star$ be an optimal solution which is specified by Lemma~\ref{Lemma:dual-bound-infinity}. Then, the following inequality holds true:
\begin{equation*}
\varphi(\beta^t) - \varphi(\beta^\star) \leq \overline{R} E_t, \quad \textnormal{for all } t \geq 1. 
\end{equation*}
where $E_t$ is defined in Eq.~\eqref{Def:residue-sinkhorn} and $\overline{R} > 0$ is defined as
\begin{equation*}
\overline{R} \mydefn \frac{\|C\|_\infty}{\eta} - \log\left(\min_{1 \leq i \leq m, 1 \leq j \leq n} r_{ij}\right).
\end{equation*}
\end{lemma}
\begin{algorithm}[!t]\small
\caption{Approximating MOT by Algorithm~\ref{Algorithm:Sinkhorn} and~\ref{Algorithm:round}} \label{Algorithm:ApproxMOT_Sinkhorn}
\begin{algorithmic}
\STATE \textbf{Input:} $\eta = \frac{\varepsilon}{2 m \log(n)}$ and $\varepsilon'=\frac{\varepsilon}{8 \left\|C\right\|_\infty}$. 
\STATE \textbf{Step 1:} Let $\tilde{r}_k \in \Delta_n$ for $\forall k \in [m]$ be defined as
\begin{equation*}
\left(\tilde{r}_1, \tilde{r}_2, \ldots, \tilde{r}_m\right) = \left(1 - \frac{\varepsilon'}{4m}\right)(r_1, r_2, \ldots, r_m) + \frac{\varepsilon'}{4mn}(\one_n, \one_n, \ldots, \one_n).   
\end{equation*}
\STATE \textbf{Step 2:} Compute $\widetilde{X} = \textsc{MultiSinkhorn}(C, \eta, \{\tilde{r}_k\}_{k \in [m]}, \varepsilon'/2)$. 
\STATE \textbf{Step 3:} Round $\widehat{X} = \textsc{Round}(\widetilde{X}, \{\tilde{r}_k\}_{k \in [m]})$. 
\STATE \textbf{Output:} $\widehat{X}$.  
\end{algorithmic}
\end{algorithm}
\begin{proof} 
We first prove that the following inequality holds true, 
\begin{equation}\label{claim-sinkhorn-objgap-main}
\begin{array}{rcl}
\max\limits_{1 \leq i \leq m}\left\{\max\limits_{1 \leq j \leq n} \beta_{ij}^t - \min\limits_{1 \leq j \leq n} \beta_{ij}^t\right\} & \leq & \overline{R}, \\
\max\limits_{1 \leq i \leq m}\left\{\max\limits_{1 \leq j \leq n} \beta_{ij}^\star - \min\limits_{1 \leq j \leq n} \beta_{ij}^\star\right\} & \leq & \overline{R}.
\end{array}
\end{equation}
Note that the second inequality is a straightforward deduction of Eq.~\eqref{claim-dual-bound-second} in the proof of Lemma~\ref{Lemma:dual-bound-infinity}. Thus, it suffices to prove the first inequality. 

We establish this by an induction argument. Indeed, this inequality holds trivially when $t=0$. Assume that this inequality holds true for $t \leq T$. By the update for $\beta$ in Algorithm~\ref{Algorithm:Sinkhorn}, $\beta_k^{T+1} = \beta_k^T$ for all $k \neq K$, where $K = \argmax_{1 \leq k \leq m} \rho(r_k, r_k(B(\beta^t)))$. This implies that 
\begin{equation*}
\max_{1 \leq j \leq n} \beta_{kj}^{T+1} - \min_{1 \leq j \leq n} \beta_{kj}^{T+1} \leq \overline{R} \textnormal{ for all } k \neq K. 
\end{equation*}
Now it remains to show $\max_{1 \leq j \leq n} \beta_{Kj}^{T+1} - \min_{1 \leq j \leq n} \beta_{Kj}^{T+1} \leq \overline{R}$. For any $l \in [n]$, we derive from the update formula of $\beta_{Kl}^{T+1}$ that 
\begin{equation*}
e^{\beta_{Kl}^{T+1}}\sum_{1 \leq i_k \leq n, \forall k \neq K} e^{\sum_{k \neq K} \beta_{ki_k}^T - \eta^{-1}C_{i_1 \ldots l \ldots i_m}} = r_{Kl} \geq \min_{1 \leq i \leq m, 1 \leq j \leq n} r_{ij}.   
\end{equation*}
Since $C$ is a nonnegative cost tensor, we derive from the above inequality that 
\begin{equation}\label{inequality-sinkhorn-objgap-fourth}
\beta_{Kl}^{T+1} \ \geq \ \log \left(\min_{1 \leq i \leq m, 1 \leq j \leq n} r_{ij}\right) - \log \left( \sum \limits_{1 \leq i_k \leq n, \forall k \neq K} e^{\sum_{k \neq K} \beta_{ki_k}^T} \right).
\end{equation}
Since $r_{Kl} \leq 1$ and $C_{i_1 \ldots i_m} \leq \|C\|_\infty$, we have
\begin{equation}\label{inequality-sinkhorn-objgap-fifth}
\beta_{Kl}^{T+1} \leq \frac{\|C\|_\infty}{\eta} - \log \left(\sum\limits_{1 \leq i_k \leq n, \forall k \neq K} e^{\sum_{k \neq K} \beta_{ki_k}^T} \right).
\end{equation}
Combining the bounds~\eqref{inequality-sinkhorn-objgap-fourth} and~\eqref{inequality-sinkhorn-objgap-fifth} implies the desired result. 

Then, we proceed to the proof of Lemma~\ref{Lemma:sinkhorn-objgap}. Since the function $\varphi$ is convex and $\beta^\star$ is an optimal solution, Eq.~\eqref{claim-sinkhorn-objgap-main} implies that 
\begin{equation*}
\varphi(\beta^t) - \varphi(\beta^\star) \leq (\beta^t - \beta^\star)^\top\nabla\varphi(\beta^t) = \sum_{k=1}^m (\beta_k^t-\beta_k^\star)^\top\left(\frac{r_k(B(\beta^t))}{\|B(\beta^t)\|_1}-r_k\right). 
\end{equation*}
Note that the initialization and the main update for the variable $\beta$ in Algorithm~\ref{Algorithm:Sinkhorn} imply that $\|B(\beta^t)\|_1=1$ for all $t \geq 1$. Thus, we have
\begin{equation}\label{inequality-sinkhorn-objgap-first}
\varphi(\beta^t) - \varphi(\beta^\star) \leq \sum_{k=1}^m (\beta_k^t-\beta_k^\star)^\top(r_k(B(\beta^t))-r_k). 
\end{equation}
Furthermore, we have $\one_n^\top r_k(B(\beta^t) = \one_n^\top r_k = 1$ for all $k \in [m]$. This implies
\begin{equation}\label{inequality-sinkhorn-objgap-second}
\one_n^\top(r_k(B(\beta^t) - r_k) = 0 \textnormal{ for all } k \in [m]. 
\end{equation}
For all $k \in [m]$, we define $2m$ shift terms as follows, 
\begin{equation*}
\Delta\beta_k^t = \frac{\max_{1 \leq j \leq n} \beta_{kj}^t + \min_{1 \leq j \leq n} \beta_{kj}^t}{2}, \qquad \Delta\beta_i^\star = \frac{\max_{1 \leq j \leq n} \beta_{kj}^\star + \min_{1 \leq j \leq n} \beta_{kj}^\star}{2}.    
\end{equation*}
Using these shift terms, we derive that 
\begin{eqnarray*}
\lefteqn{(\beta_k^t - \beta_k^\star)^\top(r_k(B(\beta^t)) - r_k)} \\
& \overset{~\eqref{inequality-sinkhorn-objgap-second}}{=} & (\beta_k^t - \Delta\beta_k^t\one_n)^\top(r_k(B(\beta^t)) - r_k) - (\beta_k^\star - \Delta\beta_k^\star\one_n)^\top(r_k(B(\beta^t)) - r_k) \\
& \leq & (\|\beta_k^t - \Delta\beta_k^t\one_n\|_\infty + \|\beta_k^\star - \Delta\beta_k^\star\one_n\|_\infty)\|r_k(B(\beta^t)) - r_k\|_1 \\
& = & \left(\max_{1 \leq j \leq n} \beta_{kj}^t - \min_{1 \leq j \leq n} \beta_{kj}^t + \max_{1 \leq j \leq n} \beta_{kj}^\star - \min_{1 \leq j \leq n} \beta_{kj}^\star\right) \frac{\|r_k(B(\beta^t)) - r_k\|_1}{2}. 
\end{eqnarray*}
Plugging Eq.~\eqref{claim-sinkhorn-objgap-main} into the above inequality yields
\begin{equation}\label{inequality-sinkhorn-objgap-third}
(\beta_k^t - \beta_k^\star)^\top(r_k(B(\beta^t)) - r_k) \leq \bar{R}\|r_k(B(\beta^t)) - r_k\|_1. 
\end{equation}
Combining Eq.~\eqref{inequality-sinkhorn-objgap-third} and Eq.~\eqref{inequality-sinkhorn-objgap-first} yields 
\begin{equation*}
\varphi(\beta^t) - \varphi(\beta^\star) \leq \bar{R}\left(\sum_{k=1}^m \|r_k(B(\beta^t)) - r_k\|_1\right) = \overline{R}E_t. 
\end{equation*}
As a consequence, we obtain the conclusion of the lemma.
\end{proof}
The second lemma gives a descent inequality for the iterates generated by Algorithm~\ref{Algorithm:Sinkhorn} with a lower bound on the progress at each iteration.  
\begin{lemma}\label{Lemma:sinkhorn-progress}
Let $\{\beta^t\}_{t \geq 0}$ be the iterates generated by Algorithm~\ref{Algorithm:Sinkhorn}. Then, the following inequality holds true: 
\begin{equation}\label{inequality-sinkhorn-progress-main}
\varphi(\beta^t) - \varphi(\beta^{t+1}) \geq \frac{1}{2}\left(\frac{E_t}{m}\right)^2, \quad \textnormal{for all } t \geq 1. 
\end{equation}
\end{lemma}
\begin{proof}
We first show that
\begin{equation}\label{claim-sinkhorn-progress-main}
\varphi(\beta^t) - \varphi(\beta^{t + 1}) \geq \frac{1}{m}\left(\sum_{k=1}^m \rho(r_k, r_k(B(\beta^t)))\right).
\end{equation}
By the definition of $\varphi$, we have
\begin{equation}\label{inequality-sinkhorn-progress-first}
\varphi(\beta^t) - \varphi(\beta^{t + 1}) = \log(\|B(\beta^t)\|_1) - \log(\|B(\beta^{t+1})\|_1) - \sum_{k=1}^m (\beta_k^t - \beta_k^{t+1})^\top r_k. 
\end{equation}
From the update formula for $\beta^{t+1}$, it is clear that $\|B(\beta^t)\|_1=\|B(\beta^{t+1})\|_1=1$ for all $t \geq 1$. Therefore, we have
\begin{equation*}
\varphi(\beta^t) - \varphi(\beta^{t + 1}) = -(\beta_K^t - \beta_K^{t+1})^\top r_K = (\log(r_K) - \log(r_K(B(\beta^t))))^\top r_K. 
\end{equation*}
Since $\one_n^\top r_K = \one_n^\top r_K(B(\beta^t)) = 1$, we have $\varphi(\beta^t) - \varphi(\beta^{t+1}) = \rho(r_K, r_K(B(\beta^t)))$. Combining this equality with the fact that the $K$-th coordinate is the greedy one yields Eq.~\eqref{claim-sinkhorn-progress-main}. 

We proceed to prove Eq.~\eqref{inequality-sinkhorn-progress-main}. Indeed, by the Pinsker inequality, we have
\begin{equation*}
\rho(r_k, r_k(B(\beta^t))) \geq \frac{1}{2}\|r_k(B(\beta^t)) - r_k\|_1^2 \textnormal{ for any } k \in [m].  
\end{equation*}
Plugging this inequality into Eq.~\eqref{claim-sinkhorn-progress-main} and using the Cauchy-Schwarz inequality yields 
\begin{eqnarray*}
\lefteqn{\varphi(\beta^t) - \varphi(\beta^{t + 1}) \geq \frac{1}{2m}\left(\sum_{k=1}^m \|r_k(B(\beta^t)) - r_k\|_1^2\right)} \\
& \geq & \frac{1}{2m^2}\left(\sum_{k=1}^m \|r_k(B(\beta^t)) - r_k\|_1\right)^2=\frac{1}{2}\left(\frac{E_t}{m}\right)^2. 
\end{eqnarray*}
This completes the proof. 
\end{proof}

\subsection{Main results}
We present an upper bound for the number of iterations required by Algorithm~\ref{Algorithm:Sinkhorn}. 
\begin{theorem}\label{Theorem:sinkhorn-iteration}
Let $\{\beta^t\}_{t \geq 0}$ be the iterates generated by Algorithm~\ref{Algorithm:Sinkhorn}. The number of iterations required to reach the stopping criterion $E_t \leq \varepsilon'$ satisfies
\begin{equation}\label{inequality-sinkhorn-iteration-main}
t \leq 2 + \frac{2 m^2 \overline{R}}{\varepsilon'},
\end{equation}
where $\overline{R}$ is defined in Lemma~\ref{Lemma:sinkhorn-objgap}. 
\end{theorem}
\begin{proof} 
Let $\beta^\star$ be an optimal solution of the dual entropic regularized MOT problem considered in Lemma~\ref{Lemma:sinkhorn-objgap}. By letting the objective gap at each iteration be $\delta^t = \varphi(\beta^t) - \varphi(\beta^\star)$, we derive from Lemma~\ref{Lemma:sinkhorn-objgap} and Lemma~\ref{Lemma:sinkhorn-progress} that 
\begin{equation*}
\delta^t \leq \overline{R} E_t, \qquad \delta^t - \delta^{t+1} \geq \frac{1}{2}\left(\frac{E_t}{m}\right)^2.
\end{equation*}
Putting these pieces together with the fact that $E_t \geq \varepsilon'$ as long as the stopping criterion is not fulfilled yields
\begin{equation*}
\delta^t - \delta^{t+1} \geq \frac{1}{2}\left(\max\left\{\left(\frac{\varepsilon'}{m}\right)^2, \left(\frac{\delta^t}{m\overline{R}}\right)^2\right\}\right). 
\end{equation*}
We now apply the switching strategy to obtain the desired upper bound in Eq.~\eqref{inequality-sinkhorn-iteration-main}. Indeed, we have
\begin{equation*}
\frac{\delta^{t+1}}{2m^2\overline{R}^2} \leq \frac{\delta^t}{2 m^2 \overline{R}^2} - \frac{(\delta^t)^2}{4m^4\overline{R}^4}, \qquad \delta^{t+1} \leq \delta^t - \frac{1}{2}\left(\frac{\varepsilon'}{m}\right)^2. 
\end{equation*}
Fixing an integer $t_1 > 0$ and considering $t > t_1$, the first inequality further implies that 
\begin{equation*}
\frac{2m^2\overline{R}^2}{\delta^{t+1}} - \frac{2m^2\overline{R}^2}{\delta^t} \geq 1 \quad \Longrightarrow \quad t_1 \leq 1 + \frac{2 m^2 \overline{R}^2}{\delta^{t_1}},
\end{equation*}
and the second inequality further implies that 
\begin{equation*}
t - t_1 \leq 1 + 2(\delta^{t_1} - \delta^t)\left(\frac{m}{\varepsilon'}\right)^2 \quad \Longrightarrow \quad t \leq 1 + t_1 + 2\delta^{t_1}\left(\frac{m}{\varepsilon'}\right)^2. 
\end{equation*}
Let $s = \delta_{t_1} \leq \delta_1$, we obtain that the total number of iterations satisfies 
\begin{equation*}
t \leq \min \limits_{0 \leq s \leq \delta^1} \left\{2 + \frac{2 m^2 \overline{R}^2}{s} + 2s\left(\frac{m}{\varepsilon'}\right)^2\right\} \leq 2 + \frac{2m^2\overline{R}}{\varepsilon'}.  
\end{equation*}
This completes the proof. 
\end{proof}
Before presenting the main result on the complexity bound of Algorithm~\ref{Algorithm:ApproxMOT_Sinkhorn}, we provide the complexity bound of Algorithm~\ref{Algorithm:round} in the following theorem.
\begin{theorem}\label{Theorem:round-scheme}
Let $X \in \br^{n \times \ldots \times n}$ be a nonnegative tensor and $\{r_i\}_{i \in [m]} \subseteq \Delta^n$ be a sequence of probability vectors, Algorithm~\ref{Algorithm:round} returns a nonnegative tensor $Y \in \br^{n \times \ldots \times n}$ satisfying that $r_k(Y) = r_k$ for all $k \in [m]$ and 
\begin{equation*}
\|Y - X\|_1 \leq 2\left(\sum_{k=1}^m \|r_k(X) - r_k\|_1\right).  
\end{equation*}
\end{theorem}
\begin{proof}
By the definition of $z_k$ and the update formula for $X^{(k)}$ for all $k \in [m]$, each entry of $X^{(m)}$ is nonnegative and
\begin{equation}\label{inequality-RS-first}
\text{err}_k = r_k - r_k(X^{(m)}) \geq 0 \text{ for all } k \in [m]. 
\end{equation}
This implies that $\|\text{err}_k\|_1 = 1 - \|X^{(m)}\|_1$ for all $k \in [m]$. Thus, we derive from Eq.~\eqref{inequality-RS-first} and the update formula for $Y$ that each entry of $Y$ is nonnegative. 

Furthermore, we define $A$ by $A_{i_1 \ldots i_m} \mydefn \prod_{k=1}^m \text{err}_{k i_k}$ for all $(i_1, \ldots, i_m) \in [n] \times \cdots \times [n]$ and find that 
\begin{equation}\label{inequality-RS-second}
[r_k(A)]_j = \text{err}_{kj} \left(\sum_{1 \leq i_l \leq n, \forall l \neq k} \prod_{l \neq k} \text{err}_{l i_l}\right) = \text{err}_{kj} \prod_{l \neq k} \|\text{err}_l\|_1. 
\end{equation}
Therefore, we conclude that 
\begin{equation*}
r_k(Y) = r_k(X^{(m)}) + \frac{r_k(A)}{\|\text{err}_1\|_1^{m-1}} \overset{~\eqref{inequality-RS-second}}{=} r_k(X^{(m)}) + \text{err}_k = r_k \textnormal{ for all} \ k \in [m]. 
\end{equation*}
It remains to estimate the $\ell_1$ bound between $Y$ and $X$. Indeed, we have
\begin{equation}\label{inequality-RS-third}
\|X\|_1 - \|X^{(m)}\|_1 = \|X^{(0)}\|_1 - \|X^{(m)}\|_1 = \sum_{k=1}^m (\|X^{(k-1)}\|_1 - \|X^{(k)}\|_1). 
\end{equation}
Since $\|X^{(k-1)}\|_1 - \|X^{(k)}\|_1$ is the amount of mass removed from $X^{(k-1)}$ by rescaling the $k$th subtensor when $r_{kj}(X^{(k-1)}) \geq r_{kj}$, we have
\begin{equation*}
\|X^{(k-1)}\|_1 - \|X^{(k)}\|_1 = \one_n^\top(\max\{0, r_k(X^{(k-1)}) - r_k\}) \textnormal{ for all } k \in [m]. 
\end{equation*}
A simple calculation using the fact that $X^{(0)} = X$ shows that
\begin{equation}\label{inequality-RS-fourth}
\|X^{(0)}\|_1 - \|X^{(1)}\|_1 = \frac{1}{2}(\|r_1(X) - r_1\|_1 + \|X\|_1 - 1). 
\end{equation}
Moreover, $r_k(X^{(0)})$ is entrywise larger than $r_k(X^{(k-1)})$ for all $k \in [m]$. That is to say, $r_k(X^{(k-1)}) \leq r_k(X^{(k-2)}) \leq \ldots \leq r_k(X^{(0)}) = r_k(X)$. This implies
\begin{equation}\label{inequality-RS-fifth}
\|X^{(k)}\|_1 - \|X^{(k+1)}\|_1 \leq \|r_{k+1}(X) - r_{k+1}\|_1 \textnormal{ for all } k \in [m-1]. 
\end{equation}
Plugging Eq.~\eqref{inequality-RS-fourth} and Eq.~\eqref{inequality-RS-fifth} into Eq.~\eqref{inequality-RS-third} yields
\begin{equation}\label{inequality-RS-sixth}
\|X\|_1 - \|X^{(m)}\|_1 \leq \frac{1}{2}(\|r_1(X) - r_1\|_1 + \|X\|_1 - 1) + \sum_{k=2}^m \|r_k(X) - r_k\|_1. 
\end{equation}
By the definition of $Y$, we have
\begin{equation*}
\|X - Y\|_1 \leq \|X - X^{(m)}\|_1 + \frac{\|A\|_1}{\|\text{err}_1\|_1^{m-1}} \overset{~\eqref{inequality-RS-second}}{=} \|X - X^{(m)}\|_1 + \|\text{err}_1\|_1. 
\end{equation*}
Since $X$ is entrywise larger than $X^{(m)}$ and $\|\text{err}_1\|_1 = 1 - \|X^{(m)}\|_1$, we have
\begin{equation}\label{inequality-RS-seventh}
\|X - Y\|_1 \leq \|X\|_1 - \|X^{(m)}\|_1 + 1 - \|X^{(m)}\|_1 = 2(\|X\|_1 - \|X^{(m)}\|_1) + 1 - \|X\|_1. 
\end{equation}
Plugging Eq.~\eqref{inequality-RS-sixth} into Eq.~\eqref{inequality-RS-seventh} yields the desired result. 
\end{proof}
We are ready to present the complexity bound of Algorithm~\ref{Algorithm:ApproxMOT_Sinkhorn} for solving the MOT problem in Eq.~\eqref{prob:MOT}. Note that $\varepsilon'=\varepsilon/(8\|C\|_\infty)$ is defined using the desired accuracy $\varepsilon > 0$.
\begin{theorem}\label{Theorem:sinkhorn-MOT}
Algorithm~\ref{Algorithm:ApproxMOT_Sinkhorn} returns an $\varepsilon$-approximate multimarginal transportation plan $\widehat{X} \in \br^{n \times \ldots \times n}$ within 
\begin{equation*}
\bigO\left(\frac{m^3 n^m\|C\|_\infty^2\log(n)}{\varepsilon^2}\right)
\end{equation*}
arithmetic operations.
\end{theorem}
\begin{proof}
We first claim that 
\begin{equation}\label{claim-sinkhorn-OT-main}
\langle C, \widehat{X}\rangle - \langle C, X^\star\rangle \leq m\eta\log(n) + 4\left(\sum_{k=1}^m\|r_k(\widetilde{X}) - r_k\|_1\right)\|C\|_\infty. 
\end{equation}
where $\widetilde{X}$ is defined in \textbf{Step 2} of Algorithm~\ref{Algorithm:ApproxMOT_Sinkhorn} and $\widehat{X}$ is returned by Algorithm~\ref{Algorithm:ApproxMOT_Sinkhorn} and $X^\star$ is an optimal multimarginal transportation plan. By the definition of $\{\tilde{r}_k\}_{k \in [m]}$ and using $\sum_{k=1}^m \|r_k(\widetilde{X}) - \tilde{r}_k\|_1 \leq \varepsilon'/2$, we have
\begin{equation*}
\sum_{k=1}^m \|r_k(\widetilde{X}) - r_k\|_1 \leq \sum_{k=1}^m (\|r_k(\widetilde{X}) - \tilde{r}_k\|_1 + \|\tilde{r}_k - r_k\|_1) \leq \frac{\varepsilon'}{2} + \sum_{k=1}^m \frac{\varepsilon'}{2m} = \varepsilon'.  
\end{equation*}
Plugging the above inequality into Eq.~\eqref{claim-sinkhorn-OT-main} and using $\eta = \varepsilon/(2m\log(n))$ and $\varepsilon'=\varepsilon/(8\|C\|_\infty)$, we obtain that $\langle C, \widehat{X}\rangle - \langle C, X^\star\rangle \leq \varepsilon$. 

It remains to bound the number of iterations required by Algorithm~\ref{Algorithm:Sinkhorn} to reach $E_t \leq \varepsilon'/2$ (cf. \textbf{Step 2} of Algorithm~\ref{Algorithm:ApproxMOT_Sinkhorn}). Using Theorem~\ref{Theorem:sinkhorn-iteration}, we have
\begin{equation*}
t \leq 2 + \frac{4m^2\overline{R}}{\varepsilon'}. 
\end{equation*}
By the definition of $\overline{R}$ (cf. Lemma~\ref{Lemma:sinkhorn-objgap}), $\eta = \varepsilon/(2m\log(n))$ and $\varepsilon'=\varepsilon/(8\|C\|_\infty)$, we have
\begin{eqnarray*}
t & \leq & 2 + \frac{32m^2\|C\|_\infty}{\varepsilon}\left(\frac{\|C\|_\infty}{\eta} - \log\left(\min_{1 \leq i \leq m, 1 \leq j \leq n} \tilde{r}_{ij}\right)\right) \\ 
& \leq & 2 + \frac{32m^2\|C\|_\infty}{\varepsilon}\left(\frac{2m\log(n)\|C\|_\infty}{\varepsilon} - \log\left(\frac{\varepsilon}{32mn\|C\|_\infty}\right)\right) \\
& = & \bigO\left(\frac{m^3\|C\|_\infty^2 \log(n)}{\varepsilon^2} \right).
\end{eqnarray*}
Since each iteration of Algorithm~\ref{Algorithm:Sinkhorn} requires $\bigO(n^m)$ arithmetic operations, the total arithmetic operations required by \textbf{Step 2} of Algorithm~\ref{Algorithm:ApproxMOT_Sinkhorn} is $\bigO(m^3n^m\|C\|_\infty^2\log(n)\varepsilon^{-2})$. In addition, computing a set of vectors $\{\tilde{r}_k\}_{k \in [m]}$ requires $\bigO(mn)$ arithmetic operations and Algorithm~\ref{Algorithm:round} requires $\bigO(mn^m)$ arithmetic operations. Putting these pieces together yields that the complexity bound of Algorithm~\ref{Algorithm:ApproxMOT_Sinkhorn} is $\bigO(m^3n^m\|C\|_\infty^2\log(n)\varepsilon^{-2})$. 

\paragraph{Proof of Eq.~\eqref{claim-sinkhorn-OT-main}:} Using Theorem~\ref{Theorem:round-scheme}, we obtain that $\widehat{X}$ is a feasible solution to the MOT problem in Eq.~\eqref{prob:MOT} and
\begin{equation*}
\|\widehat{X} - \widetilde{X}\|_1 \leq 2\left(\sum_{k=1}^m \|r_k(\widetilde{X}) - r_k\|_1\right). 
\end{equation*}
This implies that 
\begin{equation}\label{inequality-sinkhorn-OT-first}
\langle C, \widehat{X}\rangle - \langle C, \widetilde{X}\rangle \leq 2\|C\|_\infty\left(\sum_{k=1}^m \|r_k(\widetilde{X}) - r_k\|_1\right). 
\end{equation}
Letting $X^\star$ be an optimal solution of the MOT problem and $\widetilde{Y}$ be the output returned by Algorithm~\ref{Algorithm:round} with an input $X^\star$ and $\{r_k(\widetilde{X})\}_{k \in [m]}$, Theorem~\ref{Theorem:round-scheme} implies
\begin{equation}\label{inequality-sinkhorn-OT-second}
\|\widetilde{Y} - X^\star\|_1 \leq 2\left(\sum_{k=1}^m \|r_k(X^\star) - r_k(\widetilde{X})\|_1\right) = 2\left(\sum_{k=1}^m \|r_k(\widetilde{X}) - r_k\|_1\right). 
\end{equation}
Since $\widetilde{X}$ is returned by Algorithm~\ref{Algorithm:Sinkhorn}, we have $\|\widetilde{X}\|_1=1$. By the optimality condition, there exists $\widetilde{\beta} \in \Rspace^{mn}$ such that $\widetilde{X} = B(\widetilde{\beta})$ and $\widetilde{\beta}$ is an optimal solution of the following problem: 
\begin{equation*}
\min \limits_{\beta_{1}, \ldots, \beta_{m} \in \br^{n}} \log(\|B(\beta_1, \ldots, \beta_m)\|_1) - \sum_{i=1}^m \beta_i^\top r_i(\widetilde{X}).  
\end{equation*}
This implies that $\widetilde{X}$ is an optimal solution of the following problem:
\begin{equation*}
\min \ \langle C, X\rangle - \eta H(X), \quad \st \ r_k(X) = r_k(\widetilde{X}) \textnormal{ for all } k \in [m].
\end{equation*}
Since $\widetilde{Y}$ is feasible for the above problem, we have $\langle C, \widetilde{X}\rangle - \eta H(\widetilde{X}) \leq \langle C, \widetilde{Y}\rangle - \eta H(\widetilde{Y})$. Using the property of entropy regularization function~\citep{Cover-2012-Elements}, we have $0 \leq H(\widetilde{X}), H(\widetilde{Y}) \leq m\log(n)$. Putting these pieces yields 
\begin{equation}\label{inequality-sinkhorn-OT-third}
\langle C, \widetilde{X}\rangle - \langle C, \widetilde{Y}\rangle \leq m\eta\log(n). 
\end{equation}
Combining Eq.~\eqref{inequality-sinkhorn-OT-second} and Eq.~\eqref{inequality-sinkhorn-OT-third} together with the H\"{o}lder inequality yields 
\begin{equation}\label{inequality-sinkhorn-OT-fourth}
\langle C, \widetilde{X}\rangle - \langle C, X^\star\rangle \leq m\eta\log(n) + 2\|C\|_\infty\left(\sum_{k=1}^m \|r_k(\widetilde{X}) - r_k\|_1\right). 
\end{equation}
Combining Eq.~\eqref{inequality-sinkhorn-OT-first} and Eq.~\eqref{inequality-sinkhorn-OT-fourth} yields
\begin{equation*}
\langle C, \widehat{X}\rangle - \langle C, X^\star\rangle \leq m\eta\log(n) + 4\|C\|_\infty\left(\sum_{k=1}^m \|r_k - r_k(\widetilde{X})\|_1\right).  
\end{equation*}
This completes the proof of Eq.~\eqref{claim-sinkhorn-OT-main}. 
\end{proof}
\begin{remark}
Theorem~\ref{Theorem:sinkhorn-MOT} demonstrates that the complexity bound of Algorithm~\ref{Algorithm:ApproxMOT_Sinkhorn} is near-linear in $n^m$, which is the number of unknown variable of the MOT problem in Eq.~\eqref{prob:MOT}. This is the best possible dependence on $n$ that we can hope for an optimization algorithm when applied to solve the general MOT problem. Further, the complexity bound has the dependence $m^3$ which seems unimprovable using the current techniques; indeed, the iteration number of Algorithm~\ref{Algorithm:Sinkhorn} is proportional to $m^2$ and the regularization parameter $\eta$ is necessarily proportional to $1/m$ such that the output returned by Algorithm~\ref{Algorithm:Sinkhorn} can be rounded to an $\varepsilon$-approximate multimarginal transport plan. 
\end{remark}
\begin{remark}
Even though~\citet{Benamou-2015-Iterative} has shown that the Sinkhorn-type algorithm can be more efficient in practice than LP solvers,  the full theoretical analysis is not given. In contrast, our theoretical analysis provides the provably efficient way to solve the MOT problem, demonstrating the importance of the greedy update rule in the Sinkhorn-type algorithm.  This leads to an algorithmic framework in which each iteration might take $O(n^m)$ number of arithmetic operations in the worst case. However, we can develop some efficient subroutines by exploiting the special structure of many MOT problems in practice and show that the required number of arithmetic operations is only polynomial in $m$ and $n$~\citep{Altschuler-2021-Wasserstein, Altschuler-2021-Hardness, Altschuler-2022-Wasserstein}.  Some of their results are based on both the algorithmic scheme and the theoretical analysis of multimarginal Sinkhorn, demonstrating the fundamental role that our analysis play in understanding the MOT problem. 
\end{remark}

%!TEX root = paper.tex
\section{Accelerating Multimarginal Sinkhorn Algorithm}\label{sec:acceleration}
In this section, we present an \emph{accelerated multimarginal Sinkhorn algorithm} for solving the entropic regularized MOT problem in Eq.~\eqref{prob:MOT_regularized}. Together with a rounding scheme, our algorithm can be used for solving the MOT problem in Eq.~\eqref{prob:MOT} and achieves a complexity bound of $\widetilde{O}(m^3 n^{m+1/3}\varepsilon^{-4/3})$, which improves that of the multimarginal Sinkhorn algorithm in terms of $1/\varepsilon$ and accelerated alternating minimization algorithm~\citep{Tupitsa-2020-Multimarginal} in terms of $n$. The proof idea comes from a novel combination of Nesterov's estimated sequence and the techniques for analyzing the multimarginal Sinkhorn algorithm.
\begin{algorithm}[!t]\small
\caption{\textsc{Accelerated MultiSinkhorn}$(C, \eta, \{\tilde{r}_k\}_{k \in [m]}, \varepsilon')$} \label{Algorithm:acceleration}
\begin{algorithmic}
\STATE \textbf{Input:} $t = 0$, $\theta_0 = 1$, $K=1$ and $\check{\beta}^0 = \tilde{\beta}^0 = \zero_n$.  
\WHILE{$E_t > \varepsilon'$}
\STATE \textbf{Step 1.} Compute $\bar{\beta}^t = (1 - \theta_t)\check{\beta}^t + \theta_t\tilde{\beta}^t$. 
\STATE \textbf{Step 2.} Compute $\tilde{\beta}^{t+1} \in \br^{mn}$ by 
\begin{equation*}
\tilde{\beta}_k^{t+1} = \tilde{\beta}_k^t - \frac{1}{m\theta_t}\left(\frac{r_k(B(\bar{\beta}^t))}{\|B(\bar{\beta}^t)\|_1} - r_k\right) \textnormal{ for all } k \in [m]. 
\end{equation*}
\STATE \textbf{Step 3.} Compute $\grave{\beta}^t = \bar{\beta}^t + \theta_t(\tilde{\beta}^{t+1} - \tilde{\beta}^t)$.
\STATE \textbf{Step 4.} Compute $\widehat{\beta}^t \in \br^{mn}$ by 
\begin{equation*}
\widehat{\beta}_k^t = \left\{\begin{array}{ll}
\grave{\beta}_k^t + \log(r_k) - \log(r_k(B(\grave{\beta}^t))), & k = K, \\ \grave{\beta}_k^t, & \text{otherwise}.
\end{array}\right.
\end{equation*} 
\STATE \textbf{Step 5.} Compute $\beta^t = \argmin\{\varphi(\beta) \mid \beta \in \{\check{\beta}^t, \widehat{\beta}^t\}\}$.  
\STATE \textbf{Step 6.} Choose the greedy coordinate $K = \argmax_{1 \leq k \leq m} \rho(r_k, r_k(B(\beta^t)))$. 
\STATE \textbf{Step 7.} Compute $\check{\beta}^{t+1} \in \br^{mn}$ by 
\begin{equation*}
\check{\beta}_k^{t+1} = \left\{\begin{array}{ll}
\beta_k^t + \log(r_k) - \log(r_k(B(\beta^t))), & k = K, \\ \beta_k^t, & \text{otherwise}.
\end{array}\right.
\end{equation*}
\STATE \textbf{Step 8.} Compute $\theta_{t+1} = \theta_t(\sqrt{\theta_t^2 + 4} - \theta_t)/2$. 
\STATE \textbf{Step 9.} Increment by $t = t + 1$. 
\ENDWHILE
\STATE \textbf{Output:} $B(\beta^t)$.  
\end{algorithmic}
\end{algorithm}

\subsection{Algorithmic procedure} 
We present the pseudocode of accelerated multimarginal Sinkhorn algorithm in Algorithm~\ref{Algorithm:acceleration}. This algorithm achieves the acceleration by using Nesterov's estimate sequences~\citep{Nesterov-2018-Lectures}. While our algorithm can be interpreted as an accelerated block coordinate descent algorithm, it is worthy noting that our algorithm is purely \textit{deterministic} and thus differs from other accelerated randomized algorithms~\citep{Nesterov-2012-Efficiency, Lin-2015-Accelerated, Fercoq-2015-Accelerated, Allen-2016-Even, Lu-2018-Accelerating, Diakonikolas-2018-Alternating} in the machine learning and optimization literature. 

\paragraph{Comments on algorithmic scheme.} Algorithm~\ref{Algorithm:acceleration} is a novel combination of Nesterov's estimate sequences, a monotone search step, the choice of greedy coordinate and two coordinate updates. Nesterov's estimate sequences (\textbf{Step 1-3}) are crucial for optimizing a dual objective function $\varphi$ faster than Algorithm~\ref{Algorithm:Sinkhorn}. The coordinate update (\textbf{Step 4}) guarantees that $\varphi(\widehat{\beta}^t) \leq \varphi(\grave{\beta}^t)$ and $\|B(\widehat{\beta}^t)\|_1=1$. The monotone search step (\textbf{Step 5}) guarantees that $\varphi(\beta^t) \leq \varphi(\widehat{\beta}^t)$. The greedy coordinate update (\textbf{Step 6-7}) guarantees that $\varphi(\check{\beta}^{t+1}) \leq \varphi(\beta^t)$ with sufficiently large progress. Similar to Algorithm~\ref{Algorithm:Sinkhorn}, the greedy rule is based on the function $\rho: \br_+^n \times \br_+^n \rightarrow \br_+$ given by: 
\begin{equation*}
\rho(a, b) = \one_n^\top(b - a) + \sum_{i=1}^n a_i\log\left(\frac{a_i}{b_i}\right).
\end{equation*}
Furthermore, we also use the same quantity as that in the multimarginal Sinkhorn algorithm to measure the per-iteration residue of Algorithm~\ref{Algorithm:acceleration}: 
\begin{equation}\label{Def:residue-acceleration}
E_t = \sum_{k=1}^m \|r_k(B(\beta^t) - r_k\|_1.
\end{equation}

\paragraph{Comments on arithmetic operations per iteration.} The most expensive step is to compute $r_k(B(\bar{\beta}^t))/\|B(\bar{\beta}^t)\|_1$ for all $k \in [m]$. Since $B(\bar{\beta}^t)$ does not have any special property, it seems difficult to design some implementation trick to reduce the dependency on $m$. Thus, the arithmetic operations per iteration is still $\bigO(mn^m)$. Note that, the accelerated alternating minimization algorithm in~\citep{Tupitsa-2020-Multimarginal} also requires $\bigO(mn^m)$ arithmetic operations per iteration.

\paragraph{Algorithm for the MOT problem.} We present the pseudocode of our main algorithm in Algorithm~\ref{Algorithm:ApproxMOT_Acceleration}, where Algorithms~\ref{Algorithm:acceleration} and~\ref{Algorithm:round} are the subroutines. The regularization parameter $\eta$ is set as before, and \textbf{Step 1} is also necessary since the accelerated multimarginal Sinkhorn algorithm is not well behaved if the marginal distributions do not have dense support.
\begin{algorithm}[!t]\small
\caption{Approximating MOT by Algorithms~\ref{Algorithm:round} and~\ref{Algorithm:acceleration}} \label{Algorithm:ApproxMOT_Acceleration}
\begin{algorithmic}
\STATE \textbf{Input:} $\eta = \frac{\varepsilon}{2m\log(n)}$ and $\varepsilon'=\frac{\varepsilon}{8\|C\|_\infty}$. 
\STATE \textbf{Step 1:} Let $\tilde{r}_k \in \Delta_n$ for $\forall k \in [m]$ be defined as
\begin{equation*}
\left(\tilde{r}_1, \tilde{r}_2, \ldots, \tilde{r}_m\right) = \left(1 - \frac{\varepsilon'}{4m}\right)(r_1, r_2, \ldots, r_m) + \frac{\varepsilon'}{4mn}(\one_n, \one_n, \ldots, \one_n).   
\end{equation*}. 
\STATE \textbf{Step 2:} Compute $\widetilde{X} = \textsc{Accelerated MultSinkhorn}(C, \eta, \{\tilde{r}_k\}_{k \in [m]}, \varepsilon'/2)$.
\STATE \textbf{Step 3:} Round $\widehat{X} = \textsc{Round}(\widetilde{X}, \{\tilde{r}_k\}_{k \in [m]})$. 
\STATE \textbf{Output:} $\widehat{X}$. 
\end{algorithmic}
\end{algorithm} 

\subsection{Technical lemmas}
We first present two technical lemmas which are essential in the analysis of Algorithm~\ref{Algorithm:acceleration}. The first lemma provides an inductive relationship on the quantity 
\begin{equation}\label{def:residue_acceleration}
\delta_t = \varphi(\check{\beta}^t) - \varphi(\beta^\star), 
\end{equation}
where $\beta^\star$ is an optimal solution of the dual entropic regularized MOT problem in Eq.~\eqref{prob:MOT_regularized_dual}. In order to facilitate the discussion, we recall Eq.~\eqref{inequality-gradient-objective} with $\|A\|_{1 \rightarrow 2} = \sqrt{m}$ as follows, 
\begin{equation}\label{smooth:dualregOT}
\varphi(\beta') - \varphi(\beta) - (\beta' - \beta)^\top\nabla\varphi(\beta) \leq \left(\frac{m}{2}\right)\|\beta' - \beta\|^2,
\end{equation}
which will be used in the proof of the first lemma. 
\begin{lemma}\label{Lemma:acceleration-descent}
Let $\{\check{\beta}^t\}_{t \geq 0}$ be the iterates generated by Algorithm~\ref{Algorithm:acceleration} and $\beta^\star$ be an optimal solution of the dual entropic regularized MOT problem. Then the quantity $\delta_t$ defined by Eq.~\eqref{def:residue_acceleration} satisfies the following inequality,  
\begin{equation*}
\delta_{t+1} \leq (1-\theta_t)\delta_t + \frac{m\theta_t^2}{2}\left(\|\beta^\star - \tilde{\beta}^t\|^2 - \|\beta^\star - \tilde{\beta}^{t+1}\|^2\right).
\end{equation*}
\end{lemma}
\begin{proof}
Using Eq.~\eqref{smooth:dualregOT} with $\beta' = \grave{\beta}^t$ and $\beta = \bar{\beta}^t$, we have
\begin{equation*}
\varphi(\grave{\beta}^t) \leq \varphi(\bar{\beta}^t) + \theta_t(\tilde{\beta}^{t+1} - \tilde{\beta}^t)^\top\nabla\varphi(\bar{\beta}^t) + \left(\frac{m\theta_t^2}{2}\right)\|\tilde{\beta}^{t+1} - \tilde{\beta}^t\|^2. 
\end{equation*}
By simple calculations, we find that
\begin{eqnarray*}
\varphi(\bar{\beta}^t) & = & (1-\theta_t)\varphi(\bar{\beta}^t) + \theta_t \varphi(\bar{\beta}^t), \\
(\tilde{\beta}^{t+1} - \tilde{\beta}^t)^\top\nabla\varphi(\bar{\beta}^t) & = & -(\tilde{\beta}^t - \bar{\beta}^t)^\top\nabla\varphi(\bar{\beta}^t) + (\tilde{\beta}^{t+1} - \bar{\beta}^t)^\top\nabla\varphi(\bar{\beta}^t). 
\end{eqnarray*}
Putting these pieces together yields that 
\begin{eqnarray}\label{claim-acceleration-descent-main}
\varphi(\grave{\beta}^t) & \leq & \theta_t\left(\underbrace{\varphi(\bar{\beta}^t) + (\tilde{\beta}^{t+1} - \bar{\beta}^t)^\top\nabla\varphi(\bar{\beta}^t) + \left(\frac{m\theta_t}{2}\right)\|\tilde{\beta}^{t+1} - \tilde{\beta}^t\|^2}_{\textnormal{I}}\right) \\
& & + \underbrace{(1-\theta_t)\varphi(\bar{\beta}^t) - \theta_t(\tilde{\beta}^t - \bar{\beta}^t)^\top\nabla\varphi(\bar{\beta}^t)}_{\textnormal{II}}. \nonumber
\end{eqnarray}
We first estimate the term $\textnormal{II}$. Indeed, it follows from the definition of $\bar{\beta}^t$ that
\begin{equation*}
- \theta_t(\tilde{\beta}^t - \bar{\beta}^t) = \theta_t\bar{\beta}^t + (1 - \theta_t)\check{\beta}^t - \bar{\beta}^t = (1-\theta_t)(\check{\beta}^t - \bar{\beta}^t). 
\end{equation*}
Using this equation and the convexity of $\varphi$, we have
\begin{equation}\label{claim-acceleration-descent-first}
\textnormal{II} = (1-\theta_t)(\varphi(\bar{\beta}^t) + (\check{\beta}^t - \bar{\beta}^t)^\top\nabla \varphi(\bar{\beta}^t)) \leq (1-\theta_t)\varphi(\check{\beta}^t). 
\end{equation}
Then we proceed to estimate the term $\textnormal{I}$. Indeed, by the update formula for $\tilde{\beta}^{t+1}$ and the definition of $\varphi$, we have
\begin{equation*}
(\beta - \tilde{\beta}^{t+1})^\top(\nabla\varphi(\bar{\beta}^t) + m\theta_t(\tilde{\beta}^{t+1} - \tilde{\beta}^t)) = 0 \textnormal{ for all } \beta \in \br^{mn}. 
\end{equation*}
Letting $\beta =\beta^\star$ and rearranging the resulting equation yields that 
\begin{equation*}
(\tilde{\beta}^{t+1} - \bar{\beta}^t)^\top\nabla\varphi(\bar{\beta}^t) = (\beta^\star - \bar{\beta}^t)^\top\nabla \varphi(\bar{\beta}^t) + \frac{m\theta_t}{2}\left(\|\beta^\star - \tilde{\beta}^t\|^2 - \|\beta^\star - \tilde{\beta}^{t+1}\|^2-\|\tilde{\beta}^{t+1} - \tilde{\beta}^t\|^2\right).  
\end{equation*}
Using the convexity of $\varphi$ again, we have $(\beta^\star - \bar{\beta}^t)^\top\nabla \varphi(\bar{\beta}^t) \leq \varphi(\beta^\star) - \varphi(\bar{\beta}^t)$. Putting these pieces together yields that 
\begin{equation}\label{claim-acceleration-descent-second}
\textnormal{I} \leq \varphi(\beta^\star) + \frac{m\theta_t}{2}\left(\|\beta^\star - \tilde{\beta}^t\|^2 - \|\beta^\star - \tilde{\beta}^{t+1}\|^2\right). 
\end{equation}
Plugging Eq.~\eqref{claim-acceleration-descent-first} and Eq.~\eqref{claim-acceleration-descent-second} into Eq.~\eqref{claim-acceleration-descent-main} yields that 
\begin{equation*}
\varphi(\grave{\beta}^t) \leq (1-\theta_t)\varphi(\check{\beta}^t) + \theta_t\varphi(\beta^\star) + \frac{m\theta_t^2}{2}\left(\|\beta^\star - \tilde{\beta}^t\|^2 - \|\beta^\star - \tilde{\beta}^{t+1}\|^2\right). 
\end{equation*}
Since $\check{\beta}^{t+1}$ is obtained by an coordinate update from $\beta^t$, we have $\varphi(\beta^t) \geq \varphi(\check{\beta}^{t+1})$. By the definition of $\beta^t$, we have $\varphi(\widehat{\beta}^t) \geq \varphi(\beta^t)$. Since $\widehat{\beta}^t$ is obtained by an coordinate update from $\grave{\beta}^t$, we have $\varphi(\grave{\beta}^t) \geq \varphi(\widehat{\beta}^t)$. Putting these pieces together with yields that 
\begin{equation*}
\varphi(\check{\beta}^{t+1}) - \varphi(\beta^\star) \leq (1-\theta_t)(\varphi(\check{\beta}^t) - \varphi(\beta^\star)) + \frac{m\theta_t^2}{2}\left(\|\beta^\star - \tilde{\beta}^t\|^2 - \|\beta^\star - \tilde{\beta}^{t+1}\|^2\right).
\end{equation*}
This completes the proof. 
\end{proof}
The second lemma provides an upper bound for $\delta_t$ defined by Eq.~\eqref{def:residue_acceleration} where $\{\check{\beta}^t\}_{t \geq 0}$ are generated by Algorithm~\ref{Algorithm:acceleration} and $\beta^\star$ is an optimal solution defined by Corollary~\ref{Corollary:dual-bound-l2}. Note that our lemma is a direct corollary of the analysis provided in~\citet{Tseng-2008-Accelerated} and we provide the proof details for the sake of completeness. 
\begin{lemma}\label{Lemma:acceleration-bound}
Let $\{\check{\beta}^t\}_{t \geq 0}$ be the iterates generated by Algorithm~\ref{Algorithm:acceleration} and $\beta^\star$ be an optimal solution of the dual entropic regularized MOT problem satisfying that $\|\beta^\star\| \leq \sqrt{mn}R$ where $R$ is defined in Corollary~\ref{Corollary:dual-bound-l2}. Then the quantity $\delta_t$ defined by Eq.~\eqref{def:residue_acceleration} satisfies the following inequality,  
\begin{equation*}
\delta_t \leq \frac{2m^2nR^2}{(t+1)^2}.  
\end{equation*}
\end{lemma}
\begin{proof}
By simple calculations, we derive from the definition of $\theta_t$ that $(\theta_{t+1}/\theta_t)^2 = 1 - \theta_{t+1}$. Therefore, we conclude from Lemma~\ref{Lemma:acceleration-descent} that 
\begin{equation*}
\left(\frac{1-\theta_{t+1}}{\theta_{t+1}^2}\right)\delta_{t+1} - \left(\frac{1-\theta_t}{\theta_t^2}\right)\delta_t \leq \frac{m}{2}\left(\|\beta^\star - \tilde{\beta}^t\|^2 - \|\beta^\star - \tilde{\beta}^{t+1}\|^2\right).
\end{equation*}
Equivalently, we have
\begin{equation*}
\left(\frac{1-\theta_t}{\theta_t^2}\right)\delta_t + \left(\frac{m}{2}\right)\|\beta^\star - \tilde{\beta}^t\|^2 \leq \left(\frac{1-\theta_0}{\theta_0^2}\right)\delta_0 + \left(\frac{m}{2}\right)\|\beta^\star - \tilde{\beta}^0\|^2. 
\end{equation*}
Recall that $\theta_0 = 1$ and $\tilde{\beta}^0 = \zero_{mn}$, we have $\delta_t \leq (m\theta_{t-1}^2/2)\|\beta^\star\|^2 \leq (1/2)m^2nR^2\theta_{t-1}^2$. The remaining step is to show that $0 < \theta_t \leq 2/(t+2)$. Indeed, the claim holds when $t=0$ as we have $\theta_0 = 1$. Assume that the claim holds for $t \leq t_0$, i.e., $\theta_{t_0} \leq 2/(t_0+2)$, we have
\begin{equation*}
\theta_{t_0+1} = \frac{2}{1 + \sqrt{1 + 4/\theta_{t_0}^2}} \leq \frac{2}{t_0+3}. 
\end{equation*}
Putting these pieces together yields the desired inequality for $\delta_t$. 
\end{proof}

%%%%%%%%%%%%%%%%%%%%%%%%%%%%%%%%%%%%%%%%%%%%%%%%%%%%%%%%%%%%%
\subsection{Main results}
We present an upper bound for the number of iterations required by Algorithm~\ref{Algorithm:acceleration}. 
\begin{theorem}\label{Theorem:acceleration-iteration}
Let $\{\beta^t\}_{t \geq 0}$ be the iterates generated by Algorithm~\ref{Algorithm:acceleration}. The number of iterations required to reach the stopping criterion $E_t \leq \varepsilon'$ satisfies
\begin{equation*}
t \leq 1 + 4\left(\frac{\sqrt{n}mR}{\varepsilon'}\right)^{2/3}, 
\end{equation*}
where $R > 0$ is defined in Lemma~\ref{Lemma:dual-bound-infinity}.
\end{theorem}
\begin{proof}
We first claim that 
\begin{equation}\label{claim-acceleration-iteration-main}
\varphi(\beta^t) - \varphi(\check{\beta}^{t+1}) \geq \frac{1}{2m}\left(\sum_{k=1}^m \|r_k(B(\beta^t)) - r_k\|_1^2\right).
\end{equation}
By the definition of $\varphi$, we have
\begin{equation}\label{inequality-acceleration-iteration-first}
\varphi(\beta^t) - \varphi(\check{\beta}^{t+1}) = \log(\|B(\beta^t)\|_1) - \log(\|B(\check{\beta}^{t+1})\|_1) - \sum_{k=1}^m (\beta_k^t - \check{\beta}_k^{t+1})^\top r_k. 
\end{equation}
From the update formula for $\widehat{\beta}^t$ and $\check{\beta}^{t+1}$, it is clear that $\|B(\widehat{\beta}^t)\|_1 = 1$ and $\|B(\check{\beta}^{t+1})\|_1 = 1$ for all $t \geq 0$. Then we derive from the monotone search step (cf. \textbf{Step 5}) that $\|B(\beta^t)\|_1 = 1$ for all $t \geq 1$. Therefore, we have
\begin{equation*}
\varphi(\beta^t) - \varphi(\check{\beta}^{t+1}) = -(\beta_K^t - \check{\beta}_K^{t+1})^\top r_K = (\log(r_K) - \log(r_K(B(\beta^t))))^\top r_K.  
\end{equation*}
Since $\one_n^\top r_K = \one_n^\top r_K(B(\beta^t)) = 1$, we have $\varphi(\beta^t) - \varphi(\check{\beta}^{t+1}) = \rho(r_K, r_K(B(\beta^t)))$ for all $t \geq 1$. Combining this inequality with the fact that the $K$-th coordinate is the greedy one yields 
\begin{equation*}
\varphi(\beta^t) - \varphi(\check{\beta}^{t+1}) \geq \frac{1}{m}\left(\sum_{k=1}^m \rho(r_k, r_k(B(\beta^t)))\right).
\end{equation*}
Using the Pinsker inequality~\citep{Cover-2012-Elements}, we derive Eq.~\eqref{claim-acceleration-iteration-main} as desired.  

By the definition of $\beta^t$, we have $\varphi(\check{\beta}^t) \geq \varphi(\beta^t)$. Plugging this inequality into Eq.~\eqref{claim-acceleration-iteration-main} together with the Cauchy-Schwarz inequality yields 
\begin{equation*}
\varphi(\check{\beta}^t) - \varphi(\check{\beta}^{t+1}) \geq \frac{1}{2m}\left(\sum_{k=1}^m \|r_k(B(\beta^t)) - r_k\|_1^2\right) \geq \frac{1}{2}\left(\frac{E_t}{m}\right)^2.  
\end{equation*}
Therefore, we conclude that 
\begin{equation*}
\varphi(\check{\beta}^j) - \varphi(\check{\beta}^{t+1}) \geq \frac{1}{2m^2}\left(\sum_{i=j}^t E_i^2\right) \textnormal{ for any } j \in \{1, 2, \ldots, t\}.   
\end{equation*}
Since $\varphi(\check{\beta}^{t+1}) \geq \varphi(\beta^\star)$ for all $t \geq 1$, we have $\varphi(\check{\beta}^j) - \varphi(\check{\beta}^{t+1}) \leq \delta_j$. Then Lemma~\ref{Lemma:acceleration-bound} implies
\begin{equation*}
\sum_{i=j}^t E_i^2 \leq \frac{4m^4nR^2}{(j+1)^2}.
\end{equation*}
Putting these pieces together with the fact that $E_t \geq \varepsilon'$ as soon as the stopping criterion is not fulfilled yields 
\begin{equation*}
\frac{4m^4nR^2}{(j+1)^2(t-j+1)} \geq (\varepsilon')^2.  
\end{equation*}
Since this inequality holds true for all $j \in \{1, 2, \ldots, t\}$, we assume without loss of generality that $t$ is even and let $j = t/2$. Then, we obtain that 
\begin{equation*}
t \leq 1 + 4\left(\frac{\sqrt{n}m^2R}{\varepsilon'}\right)^{2/3}. 
\end{equation*}
This completes the proof. 
\end{proof}
We are ready to present the complexity bound of Algorithm~\ref{Algorithm:ApproxMOT_Acceleration} for solving the MOT problem in Eq.~\eqref{prob:MOT}. Note that $\varepsilon'=\varepsilon/(8\|C\|_\infty)$ is defined using the desired accuracy $\varepsilon > 0$. 
\begin{theorem}\label{Theorem:acceleration-MOT}
Algorithm~\ref{Algorithm:ApproxMOT_Acceleration} returns an $\varepsilon$-approximate multimarginal transportation plan $\widehat{X} \in \br^{n \times \ldots \times n}$ within 
\begin{equation*}
\bigO\left(\frac{m^3 n^{m+1/3}\|C\|_\infty^{4/3}(\log(n))^{1/3}}{\varepsilon^{4/3}}\right)
\end{equation*}
arithmetic operations.
\end{theorem}
\begin{proof}
Applying the same argument which is used in Theorem~\ref{Theorem:sinkhorn-MOT}, we obtain that $\langle C, \widehat{X}\rangle - \langle C, X^\star\rangle \leq \varepsilon$ where $\widehat{X}$ is returned by Algorithm~\ref{Algorithm:ApproxMOT_Acceleration}. 

It remains to bound the number of iterations required by Algorithm~\ref{Algorithm:acceleration} to reach the criterion $E_t \leq \varepsilon'/2$ (cf. \textbf{Step 2} in Algorithm~\ref{Algorithm:ApproxMOT_Acceleration}). Using Theorem~\ref{Theorem:acceleration-iteration}, we have
\begin{equation*}
t \leq 1 + 4\left(\frac{\sqrt{n}mR}{\varepsilon'}\right)^{2/3}.  
\end{equation*}
By the definition of $R$ (cf. Lemma~\ref{Lemma:dual-bound-infinity}), $\eta = \varepsilon/(2m\log(n))$ and $\varepsilon'=\varepsilon/(8\|C\|_\infty)$, we have 
\begin{eqnarray*}
t & \leq & 1 + 4\left(\frac{\sqrt{n}m^2R}{\varepsilon'}\right)^{2/3} \\
& \leq & 1 + 4\left[\frac{8\sqrt{n}m^2\|C\|_\infty}{\varepsilon}\left(\frac{\left\|C\right\|_\infty}{\eta} + (m-1)\log(n) - 2\log\left(\min_{1 \leq i \leq m, 1 \leq j \leq n} \tilde{r}_{ij}\right)\right)\right]^{2/3} \\
& \leq & 1 + 4\left[\frac{8\sqrt{n}m^2\|C\|_\infty}{\varepsilon}\left(\frac{2m\log(n)\left\|C\right\|_\infty}{\varepsilon} + (m-1)\log(n) - 2\log\left(\frac{\varepsilon}{32mn\|C\|_\infty}\right)\right)\right]^{2/3} \\
& = & \bigO\left(\frac{m^2 n^{1/3} \|C\|_\infty^{4/3}(\log(n))^{1/3}}{ \varepsilon^{4/3}}\right). 
\end{eqnarray*}
Since each iteration of Algorithm~\ref{Algorithm:acceleration} requires $\bigO(mn^m)$ arithmetic operations, the total arithmetic operations required by \textbf{Step 2} of Algorithm~\ref{Algorithm:ApproxMOT_Acceleration} is $\bigO(m^3n^{m+1/3}\|C\|_\infty^{4/3}(\log(n))^{1/3}\varepsilon^{-4/3})$. In addition, computing a set of vectors $\{\tilde{r}_k\}_{k \in [m]}$ requires $\bigO(mn)$ arithmetic operations and Algorithm~\ref{Algorithm:round} requires $\bigO(mn^m)$ arithmetic operations. Putting these pieces together yields that the complexity bound of Algorithm~\ref{Algorithm:ApproxMOT_Acceleration} is $\bigO(m^3n^{m+1/3}\|C\|_\infty^{4/3}(\log(n))^{1/3}\varepsilon^{-4/3})$. 
\end{proof}
\begin{remark}
Theorem~\ref{Theorem:acceleration-MOT} demonstrates that the complexity bound of Algorithm~\ref{Algorithm:ApproxMOT_Acceleration} is better than that of Algorithm~\ref{Algorithm:ApproxMOT_Sinkhorn} in terms of $1/\varepsilon$ but not near-linear in $n^m$. To be more specific, Algorithm~\ref{Algorithm:ApproxMOT_Acceleration} is recommended when $n \in (0, 1/\varepsilon^2)$. This occurs if the desired solution accuracy is relatively small, saying $10^{-4}$, and the examples include the application problems from economics, physics and generalized Euler flows. In contrast, Algorithm~\ref{Algorithm:ApproxMOT_Sinkhorn} is recommended when $n \in (1/\varepsilon^2, +\infty)$. This occurs if the desired solution accuracy is relatively large, saying $10^{-2}$, and the examples include the application problems from image processing. 
\end{remark}
\begin{remark}
The complexity bound has the same dependence $m^3$ as that of Algorithm~\ref{Algorithm:ApproxMOT_Sinkhorn}. However, the improvement seems possible and can be achieved if we implement \textbf{Step 2} of Algorithm~\ref{Algorithm:acceleration} in distributed parallel manner and choose the greedy coordinate in \textbf{Step 6} using the implementation trick we have mentioned before. Each iteration of Algorithm~\ref{Algorithm:acceleration} requires $\bigO(n^m)$ arithmetic operations and thus Algorithm~\ref{Algorithm:ApproxMOT_Acceleration} achieves the complexity bound of $\bigO(m^2n^{m+1/3}\|C\|_\infty^{4/3}(\log(n))^{1/3}\varepsilon^{-4/3})$. Further, it seems possible to improve the dependence of $n$ by extending other algorithmic frameworks to the MOT setting~\citep{Blanchet-2018-Towards, Lahn-2019-Graph, Jambulapati-2019-Direct}. However, such extension is challenging since we are not clear whether these frameworks heavily depend on the minimum-cost flow structure of the OT problem or not. As such, we leave this topic to the future work. 
\end{remark}

%!TEX root = paper.tex
\section{Experiments}\label{sec:experiments}
In this section, we evaluate our new algorithms on both synthetic data and real images. In particular, we compute the free-support Wasserstein barycenter based on the OT distance with the quadratic Euclidean distance ground cost function and compare our algorithms with the commercial linear programming (LP) solver \textsc{Gurobi}. All the experiments are conducted in MATLAB R2020a on a workstation with an Intel Core i5-9400F (6 cores and 6 threads) and 32GB memory, equipped with Ubuntu 18.04.
\begin{figure*}[!t]
\begin{minipage}[b]{.32\textwidth}
\includegraphics[width=55mm, height=42mm]{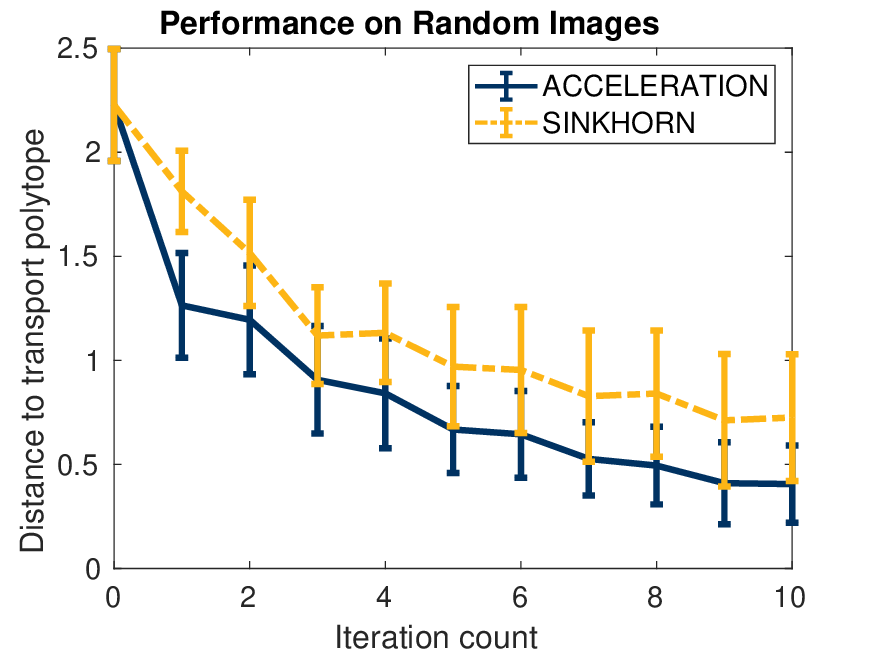}
\end{minipage}
\begin{minipage}[b]{.32\textwidth}
\includegraphics[width=55mm, height=42mm]{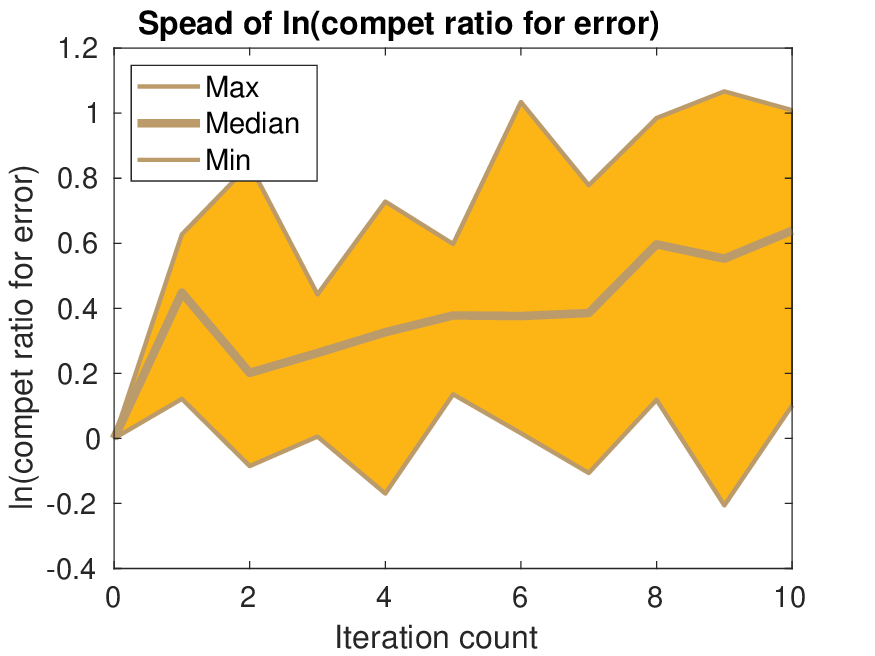}
\end{minipage}
\begin{minipage}[b]{.32\textwidth}
\includegraphics[width=55mm, height=42mm]{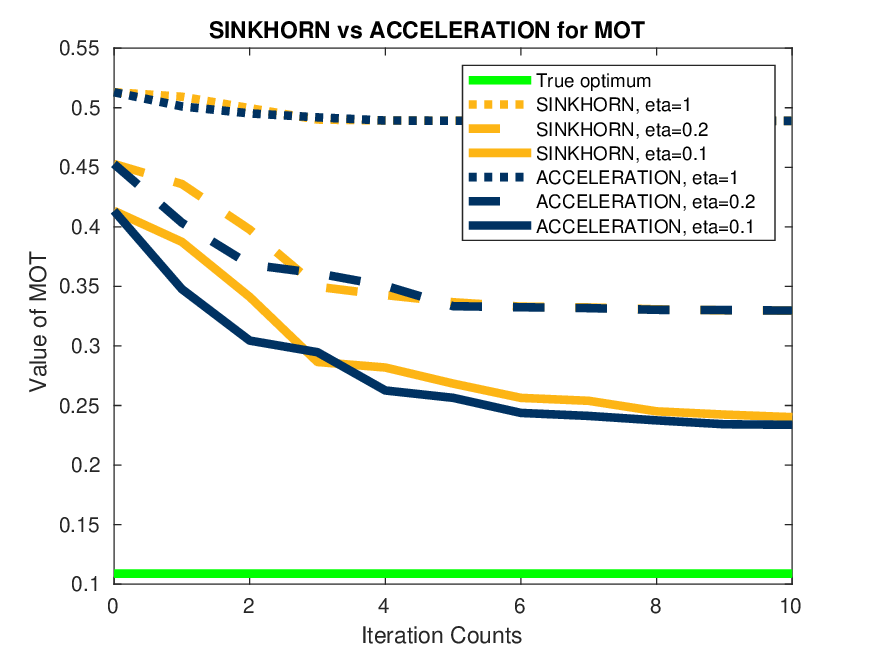}
\end{minipage}
\\
\begin{minipage}[b]{.32\textwidth}
\includegraphics[width=55mm, height=42mm]{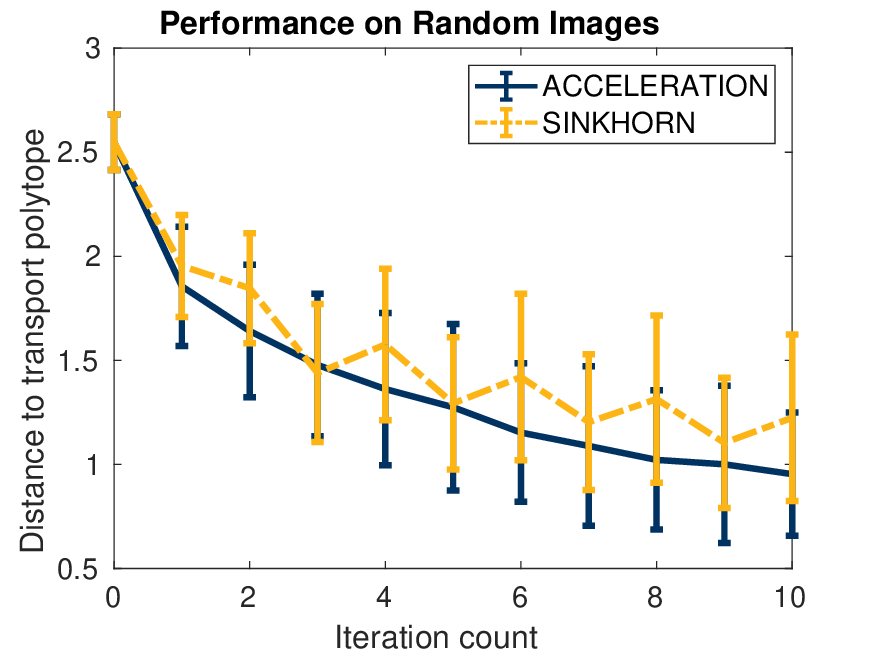}
\end{minipage}
\begin{minipage}[b]{.32\textwidth}
\includegraphics[width=55mm, height=42mm]{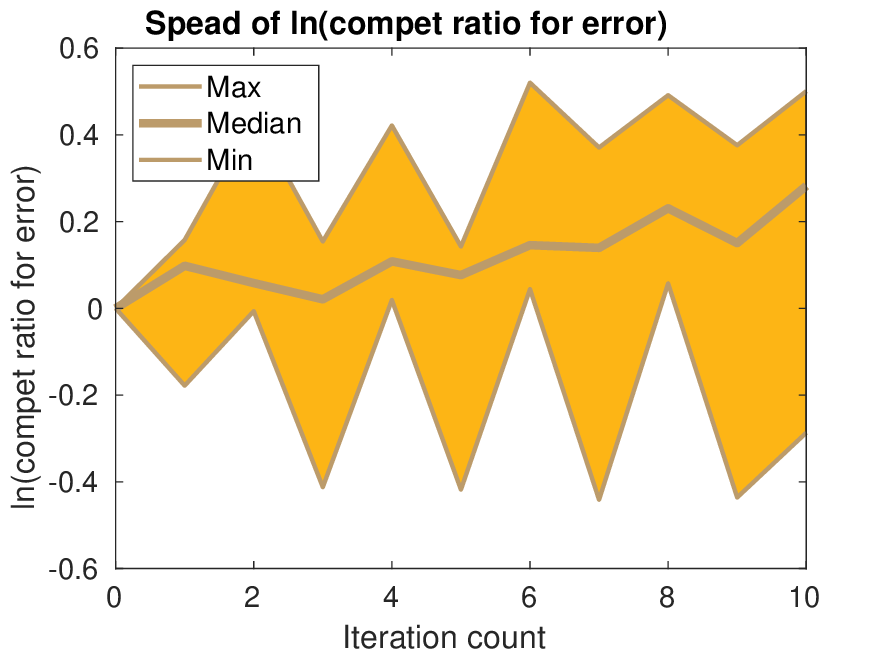}
\end{minipage}
\begin{minipage}[b]{.32\textwidth}
\includegraphics[width=55mm, height=42mm]{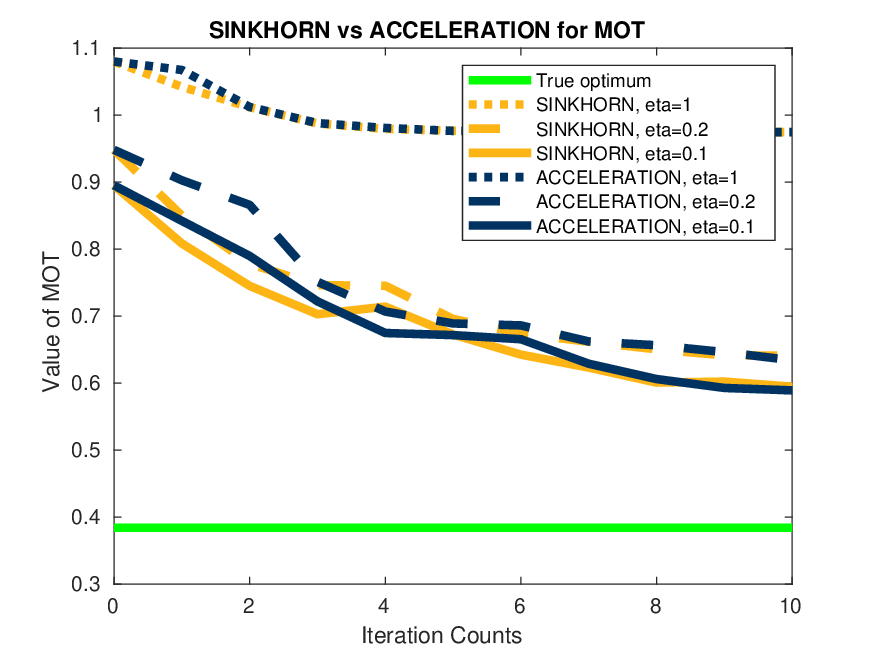}
\end{minipage}
\caption{Performance of multimarginal Sinkhorn v.s. accelerated multimarginal Sinkhorn on the randomly generated synthetic images. Number of pixel in each synthetic image is set as $n=25$ (top) and $n=100$ (bottom).}\label{fig:synthetic}
\end{figure*}
\subsection{Experiments on synthetic data}\label{subsec:synthetic}
We follow the setup in~\citet{Altschuler-2017-Near} in order to compare different algorithms on the synthetic images. More specifically, we generate a triple of random grayscale images, each normalized to have unit total mass. The marginals $r_1$, $r_2$ and $r_3$ represent three images, and the cost tensor $C$ is generated by
\begin{equation*}
C_{i_1, i_2, i_3} = \frac{1}{2}\left(\sum_{k=1}^3 \lambda_k\|x_{i_k} - A_{i_1, i_2, i_3}(x)\|^2\right) \textnormal{ for all } (i_1, i_2, i_3) \in [n] \times [n] \times [n], 
\end{equation*}
where $A_{i_1, i_2, i_3}(x) = \sum_{k=1}^3 \lambda_k x_{i_k}$ is the Euclidean barycenter and $x = \{x_i\}_{i \in [n]} \subseteq \br^d$ are pixel locations in the images. Moreover, $\lambda = (\lambda_1, \lambda_2, \lambda_3) \in \Delta^3$ is a weight vector and set as $(1/3, 1/3, 1/3)$ consistently in this subsection. 

Each of the images has $n$ pixel locations in total and is generated based on randomly positioning a foreground square in otherwise black background. We utilize a uniform distribution on $[0, 1]$ for the intensities of the background pixels and a uniform distribution on $[0, 50]$ for the foreground pixels. We set the proportion of the size of the square is as $10\%$ of the image and implement all the algorithms on the synthetic images with different size $n$. 

We generalize two metrics proposed by~\citet{Altschuler-2017-Near} and use them to quantitatively measure the performance of different algorithms. The first metric is the distance between the output of the algorithm, $X$, and the transportation polytope between the marginals $r_1$, $r_2$ and $r_3$. Formally, we have  
\begin{equation*}
d(X) = \|r_1(X) - r_1\|_1 + \|r_2(X) - r_2\| + \|r_3(X) - r_3\|_1, 
\end{equation*}
where $r_1(X)$, $r_2(X)$ and $r_3(X)$ are the marginal vectors of the output $X$ while $r_1$, $r_2$ and $r_3$ stand for the true marginal vectors. The second metric is the
competitive ratio, defined by $\log(d(X_1)/d(X_2))$ where $d(X_1)$ and $d(X_2)$ refer to the distance between the outputs of two algorithms and the transportation polytope.

\begin{figure}[!t]
\begin{center}
\includegraphics[width=0.6\textwidth]{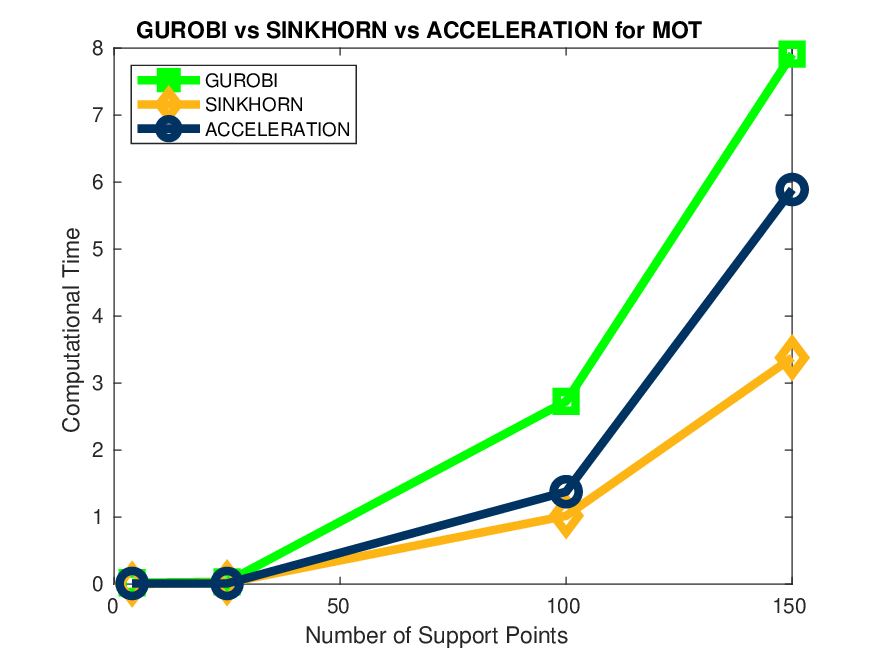}
\caption{Computational efficiency of Gurobi v.s. our algorithms as $n$ varies.}
\label{fig:gurobi}
\end{center}
\end{figure}
We perform a pairwise comparative experiment: multimarginal Sinkhorn versus accelerated multimarginal Sinkhorn, by running both algorithms with ten randomly selected pairs of synthetic images with varying size $n \in \{25, 100\}$. In order to have further evaluations with these algorithms, we also compare their performance with different choices of regularization parameter $\eta \in \{1, 0.2, 0.1\}$ while using the value of the MOT problem (without entropic regularization term) as the baseline. The maximum number of iterations is set as $10$.
\begin{figure*}[!t]
\begin{minipage}[b]{.32\textwidth}
\includegraphics[width=55mm, height=42mm]{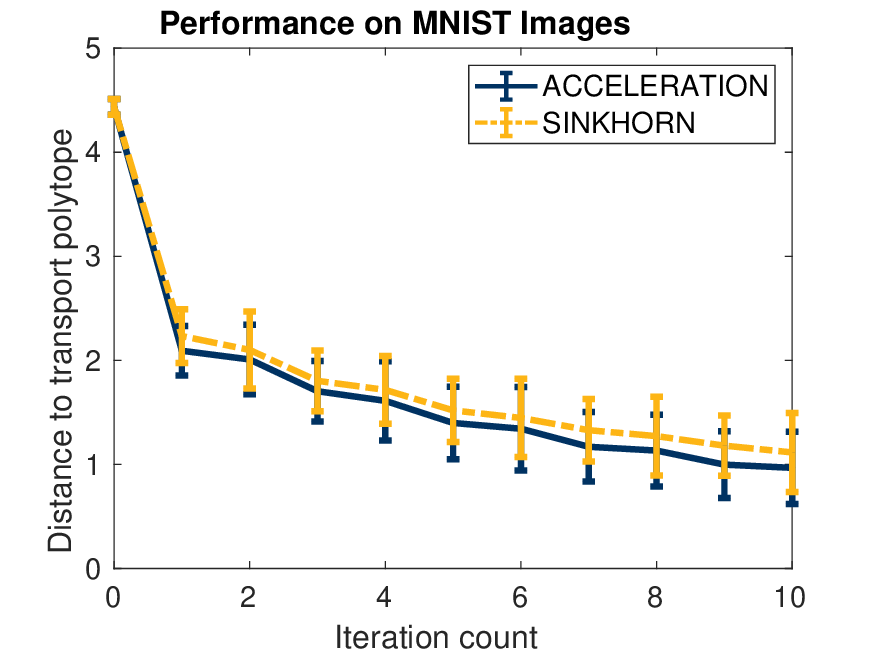}
\end{minipage}
\begin{minipage}[b]{.32\textwidth}
\includegraphics[width=55mm, height=42mm]{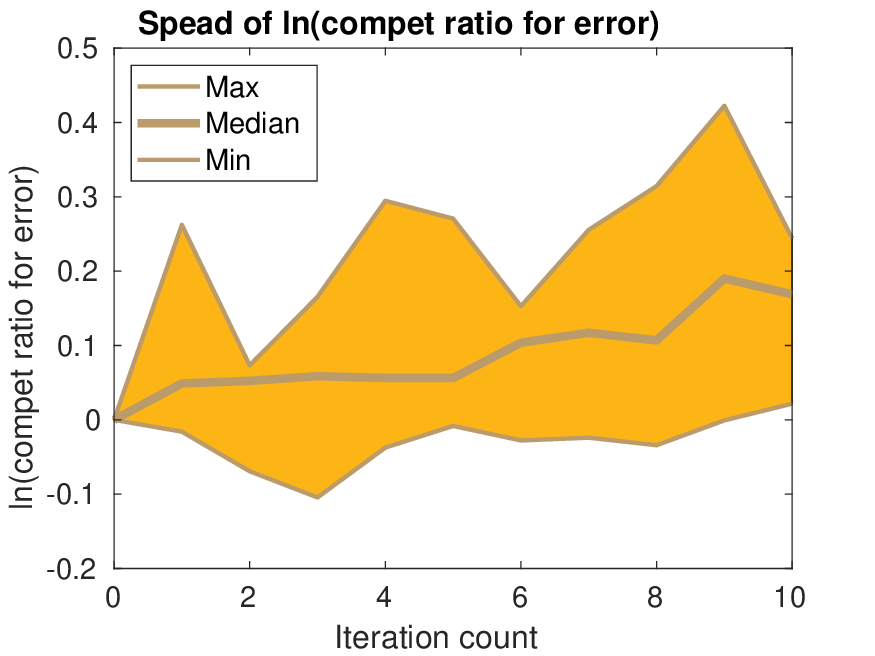}
\end{minipage}
\begin{minipage}[b]{.32\textwidth}
\includegraphics[width=55mm, height=42mm]{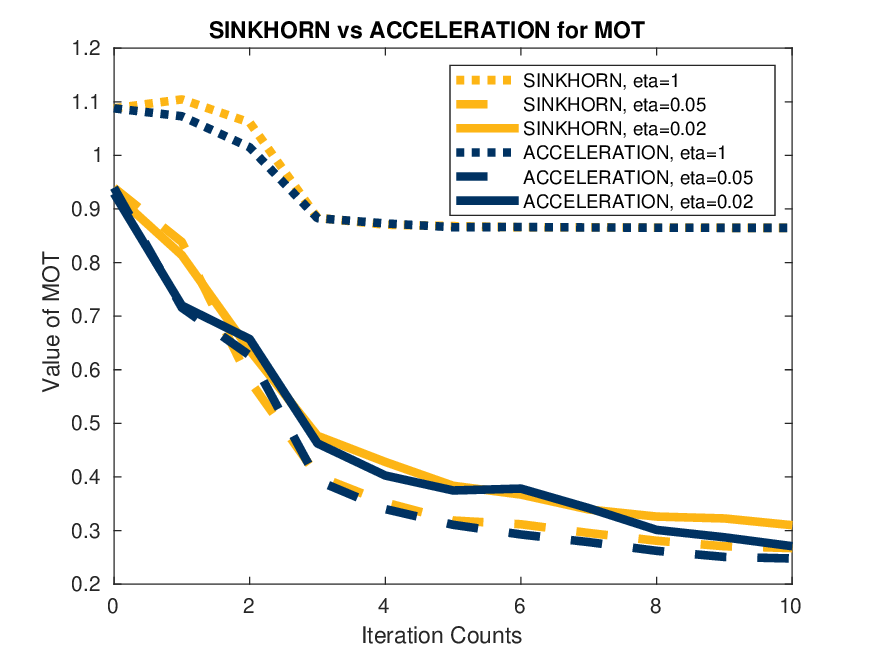}
\end{minipage}
\caption{Performance of multimarginal Sinkhorn v.s. accelerated multimarginal Sinkhorn on MNIST images. Number of pixel in each MNIST image is set as $n=576$.}\label{fig:mnist}
\end{figure*}
\paragraph{Experimental results.} Figure~\ref{fig:synthetic} summarizes the results on synthetic images. The images in the first row show the comparative performance of both algorithms in terms of the iteration counts on 10 triples of $5 \times 5$ synthetic images. In the leftmost one, the comparison uses distance to transportation polytope $d(X)$ where $X$ are returned by the algorithms. In the middle one, the maximum/median/minimum values of the competitive ratios are utilized for the comparison. In the rightmost one, we vary the regularization parameter $\eta \in \{1, 0.2, 0.1\}$ for both algorithms together with the value of the unregularized MOT problem as the baseline. It is clear that accelerated multimarginal Sinkhorn algorithm outperforms multimarginal Sinkhorn algorithm in terms of iteration numbers, illustrating the improvement achieved by using the \textit{estimated sequence} and \textit{monotone search}.

To further compare our algorithms with \textsc{Gurobi} in terms of computational efficiency, we conduct one more experiment with varying number of support points (or pixel locations) $n \in \{25, 100, 144\}$. Figure~\ref{fig:gurobi} shows the running time taken by three algorithms across a wide range of $n$. As $n$ increases, we find that multimarginal Sinkhorn algorithm performs the best, followed by accelerated multimarginal Sinkhorn algorithm, both outperforming \textsc{Gurobi}. This demonstrates that classical LP algorithms might not be suitable for solving the MOT problem, partially confirming our results in Section~\ref{sec:hardness}. Moreover, despite fewer iterations, the direct implementation of accelerated multimarginal Sinkhorn algorithm is indeed slower than multimarginal Sinkhorn algorithm. This is mainly due to the heavy computation of gradient and we believe some parallel computing toolbox can be helpful. However, this is beyond the scope of this paper and we leave it to future research.

\subsection{Experiments on real images}\label{subsec:real}
We conduct the experiment with the same setup and MNIST dataset\footnote{Available in http://yann.lecun.com/exdb/mnist/}. The MNIST dataset consists of 60,000 images of handwritten digits of size $28$ by $28$ pixels. We add a very small noise term ($10^{-6}$) to all the zero elements in the measures and then normalize them such that their sum becomes one. We also vary the regularization parameter $\eta \in \{1, 0.05, 0.02\}$ for both algorithms but cannot run \textsc{Gurobi}. Indeed, the LP constructed from the MOT problem using 3 MNIST images is so lagre that \textsc{Gurobi} is out of memory. Figure~\ref{fig:mnist} presents the comparative performance of our algorithms on the MNIST images, and we find that it is consistent with the performance on the randomly generated synthetic images. 

In order to better visualize the quality of approximate barycenters obtained by each algorithm, we run our algorithms with $\eta=0.05$ to compute the free-support Wasserstein barycenter of two triple of real images with different weight vectors. Indeed, we solve the MOT problem as before and form the barycenter as follows,  
\begin{equation*}
\mu_\lambda = \sum_{k=1}^3 \sum_{1 \leq i_k \leq n} \gamma_{i_1, i_2, i_3}\delta_{A_{i_1, i_2, i_3}(x)}, 
\end{equation*} 
where $A_{i_1, i_2, i_3}(x) = \sum_{k=1}^3 \lambda_k x_{i_k}$ is the Euclidean barycenter, and $x = \{x_i\}_{i \in [n]} \subseteq \br^d$ are pixel locations in the images and $\gamma \in \br^{n \times n \times n}$ is an optimal multimarginal transportation plan that solves the MOT problem. 

Figure~\ref{fig:barycenter} presents the approximate barycenters obtained by running our algorithms. These results demonstrate that our algorithms can successfully capture the free-support barycenters of high quality by solving the MOT problem and are at least competitive with the existing algorithms~\citep{Benamou-2015-Iterative, Benamou-2019-Generalized, Peyre-2019-Computational} in practice.  
\begin{figure}[!t]\hspace*{-3.5em}
\begin{tikzpicture}[
node distance = 0pt,
every node/.style = {minimum size=7mm, inner sep=0pt, outer sep=0pt}
                        ]
\node (n0)  {\includegraphics[width=30mm, height=20mm]{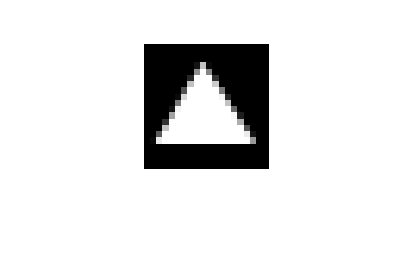}};
\node (n1) [below=of n0.south]          {};
\node (n11) [left=of n1]                {\includegraphics[width=12mm, height=12mm]{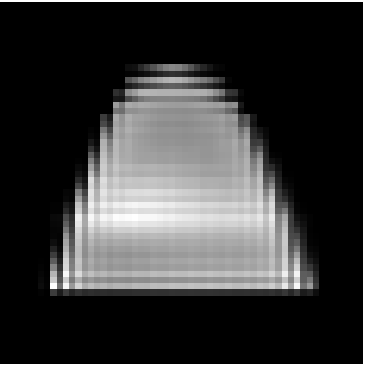}};
\node (n12) [right=of n1]              	{\includegraphics[width=12mm, height=12mm]{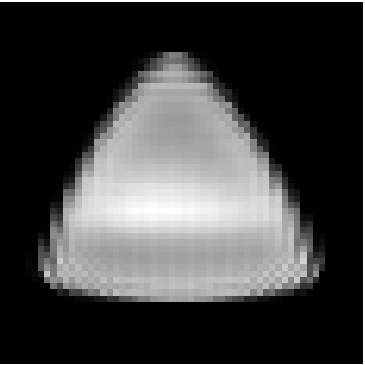}};
\node (n2) [below=of n1.south]          {};
\node (n21) [below=of n11.south]    	{};
\node (n22) [below=of n12.south]        {};
\node (n3) [below=of n2.south]          {\includegraphics[width=12mm, height=12mm]{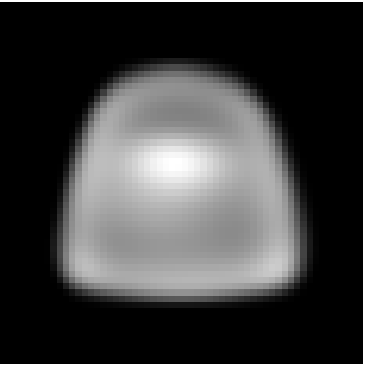}};
\node (n32) [left=of n3]                {};
\node (n33) [right=of n3]               {};
\node (n31) [left=of n32]               {\includegraphics[width=12mm, height=12mm]{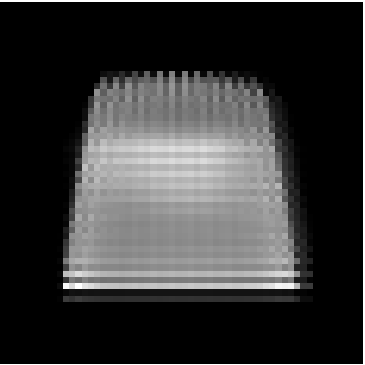}};
\node (n34) [right=of n33]              {\includegraphics[width=12mm, height=12mm]{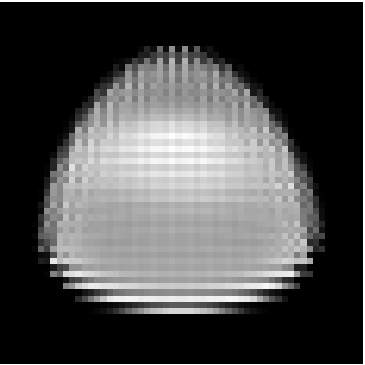}};
\node (n4) [below=of n3.south]          {};
\node (n41) [below=of n31.south]    	{};
\node (n42) [below=of n32.south]        {};
\node (n43) [below=of n33.south]    	{};
\node (n44) [below=of n34.south]        {};
\node (n5) [below=of n4.south]          {};
\node (n52) [left=of n5]                {\includegraphics[width=12mm, height=12mm]{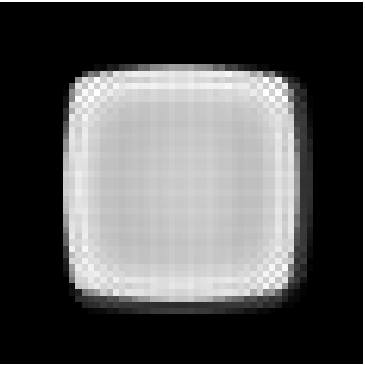}};
\node (n53) [right=of n5]               {\includegraphics[width=12mm, height=12mm]{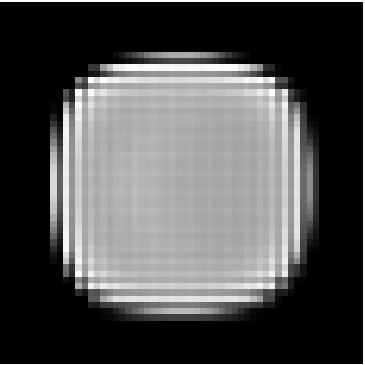}};
\node (n51) [left=of n52]               {\includegraphics[width=30mm, height=20mm]{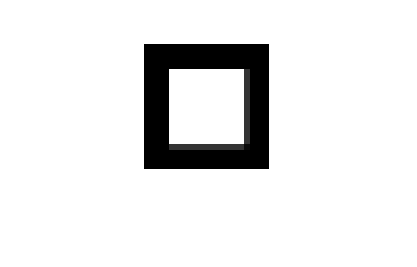}};
\node (n54) [right=of n53]              {\includegraphics[width=30mm, height=20mm]{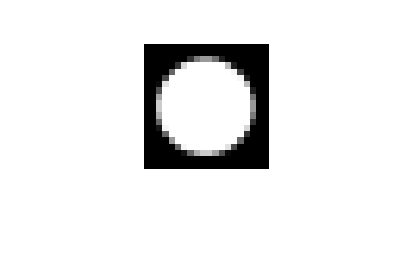}};
\end{tikzpicture}
\begin{tikzpicture}[
node distance = 0pt,
every node/.style = {minimum size=7mm, inner sep=0pt, outer sep=0pt}
                        ]
\node (n0)  {\includegraphics[width=30mm, height=20mm]{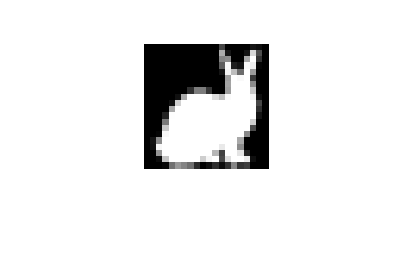}};
\node (n1) [below=of n0.south]          {};
\node (n11) [left=of n1]                {\includegraphics[width=12mm, height=12mm]{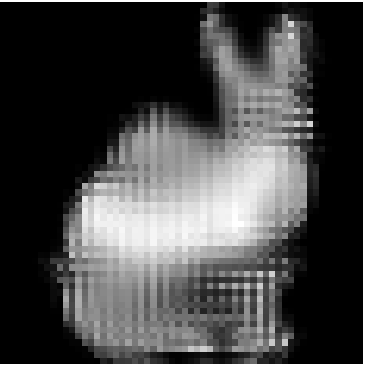}};
\node (n12) [right=of n1]              	{\includegraphics[width=12mm, height=12mm]{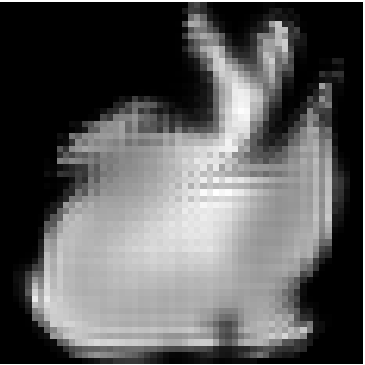}};
\node (n2) [below=of n1.south]          {};
\node (n21) [below=of n11.south]    	{};
\node (n22) [below=of n12.south]        {};
\node (n3) [below=of n2.south]          {\includegraphics[width=12mm, height=12mm]{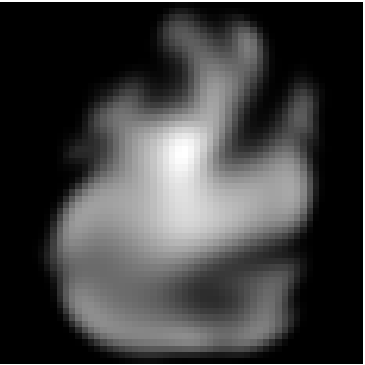}};
\node (n32) [left=of n3]                {};
\node (n33) [right=of n3]               {};
\node (n31) [left=of n32]               {\includegraphics[width=12mm, height=12mm]{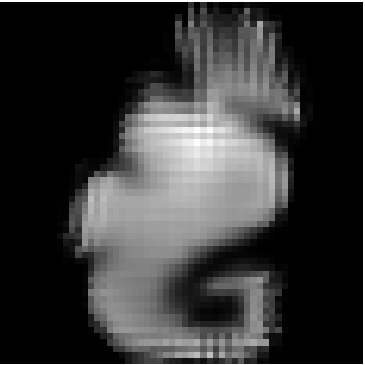}};
\node (n34) [right=of n33]              {\includegraphics[width=12mm, height=12mm]{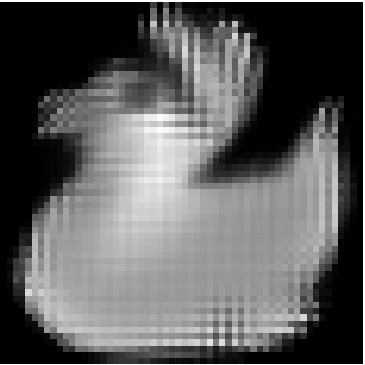}};
\node (n4) [below=of n3.south]          {};
\node (n41) [below=of n31.south]    	{};
\node (n42) [below=of n32.south]        {};
\node (n43) [below=of n33.south]    	{};
\node (n44) [below=of n34.south]        {};
\node (n5) [below=of n4.south]          {};
\node (n52) [left=of n5]                {\includegraphics[width=12mm, height=12mm]{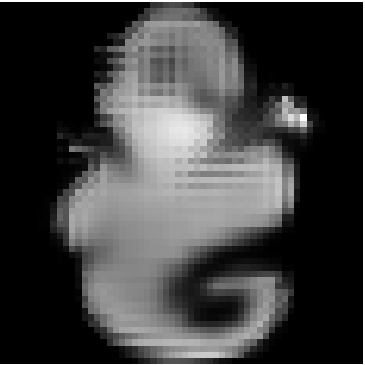}};
\node (n53) [right=of n5]               {\includegraphics[width=12mm, height=12mm]{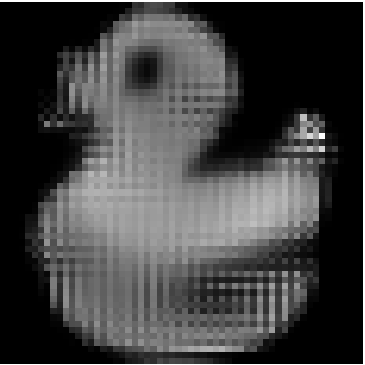}};
\node (n51) [left=of n52]               {\includegraphics[width=30mm, height=20mm]{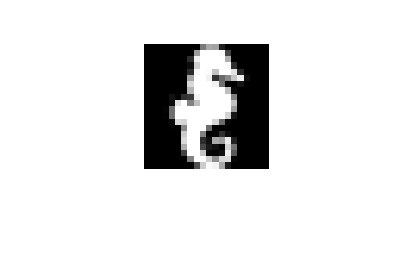}};
\node (n54) [right=of n53]              {\includegraphics[width=30mm, height=20mm]{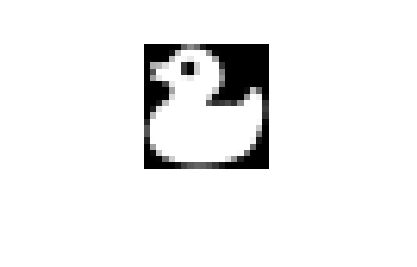}};
\end{tikzpicture} \\ \hspace*{-3.5em}
\begin{tikzpicture}[
node distance = 0pt,
every node/.style = {minimum size=7mm, inner sep=0pt, outer sep=0pt}
                        ]
\node (n0)  {\includegraphics[width=30mm, height=20mm]{figs/triangle.eps}};
\node (n1) [below=of n0.south]          {};
\node (n11) [left=of n1]                {\includegraphics[width=12mm, height=12mm]{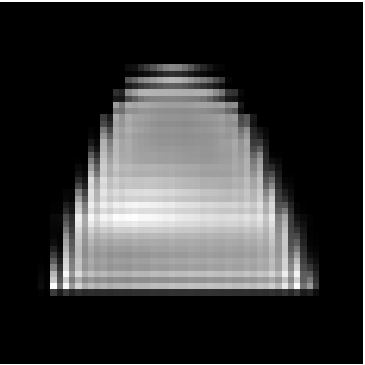}};
\node (n12) [right=of n1]              	{\includegraphics[width=12mm, height=12mm]{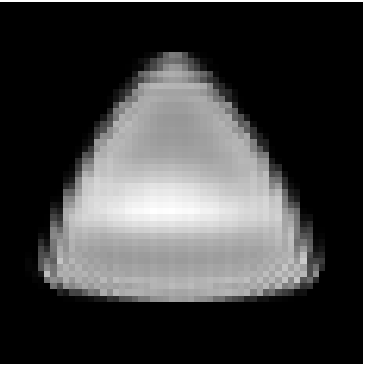}};
\node (n2) [below=of n1.south]          {};
\node (n21) [below=of n11.south]    	{};
\node (n22) [below=of n12.south]        {};
\node (n3) [below=of n2.south]          {\includegraphics[width=12mm, height=12mm]{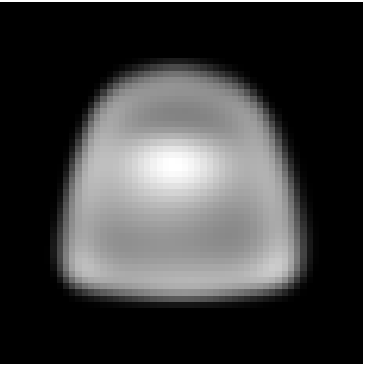}};
\node (n32) [left=of n3]                {};
\node (n33) [right=of n3]               {};
\node (n31) [left=of n32]               {\includegraphics[width=12mm, height=12mm]{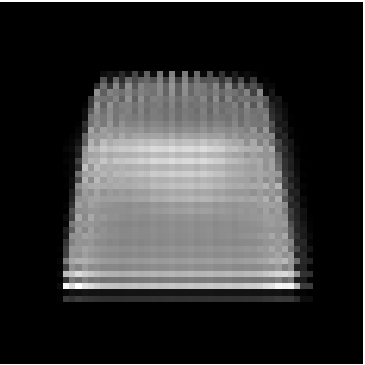}};
\node (n34) [right=of n33]              {\includegraphics[width=12mm, height=12mm]{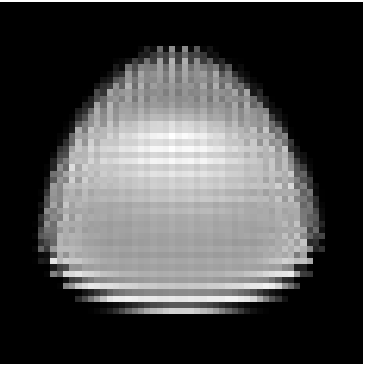}};
\node (n4) [below=of n3.south]          {};
\node (n41) [below=of n31.south]    	{};
\node (n42) [below=of n32.south]        {};
\node (n43) [below=of n33.south]    	{};
\node (n44) [below=of n34.south]        {};
\node (n5) [below=of n4.south]          {};
\node (n52) [left=of n5]                {\includegraphics[width=12mm, height=12mm]{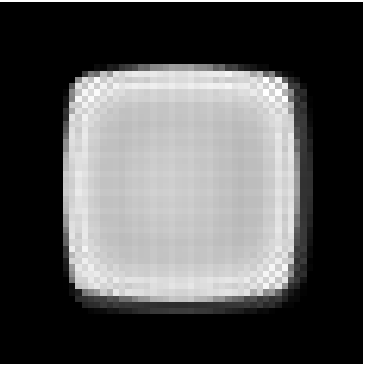}};
\node (n53) [right=of n5]               {\includegraphics[width=12mm, height=12mm]{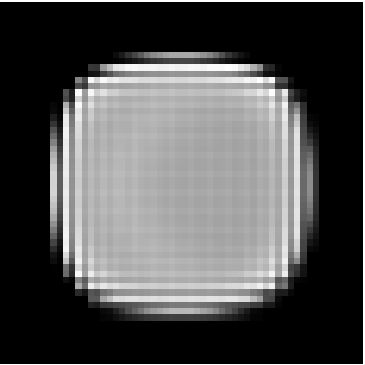}};
\node (n51) [left=of n52]               {\includegraphics[width=30mm, height=20mm]{figs/square.eps}};
\node (n54) [right=of n53]              {\includegraphics[width=30mm, height=20mm]{figs/circle.eps}};
\end{tikzpicture}
\begin{tikzpicture}[
node distance = 0pt,
every node/.style = {minimum size=7mm, inner sep=0pt, outer sep=0pt}
                        ]
\node (n0)  {\includegraphics[width=30mm, height=20mm]{figs/rabbit.eps}};
\node (n1) [below=of n0.south]          {};
\node (n11) [left=of n1]                {\includegraphics[width=12mm, height=12mm]{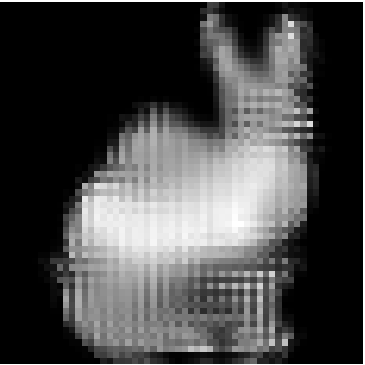}};
\node (n12) [right=of n1]              	{\includegraphics[width=12mm, height=12mm]{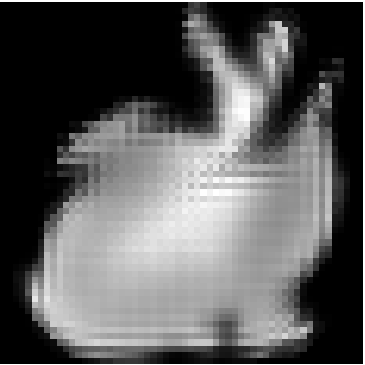}};
\node (n2) [below=of n1.south]          {};
\node (n21) [below=of n11.south]    	{};
\node (n22) [below=of n12.south]        {};
\node (n3) [below=of n2.south]          {\includegraphics[width=12mm, height=12mm]{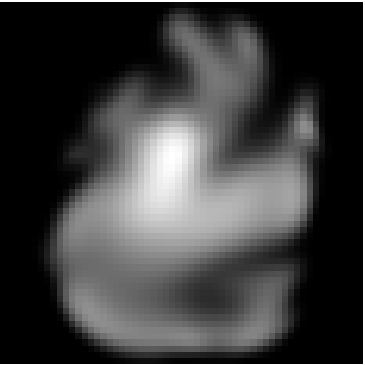}};
\node (n32) [left=of n3]                {};
\node (n33) [right=of n3]               {};
\node (n31) [left=of n32]               {\includegraphics[width=12mm, height=12mm]{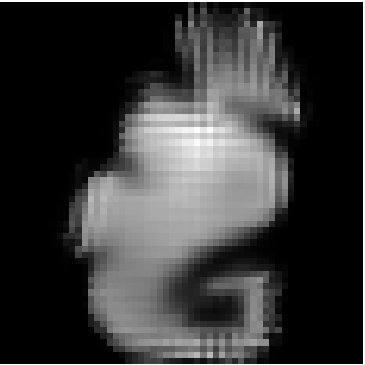}};
\node (n34) [right=of n33]              {\includegraphics[width=12mm, height=12mm]{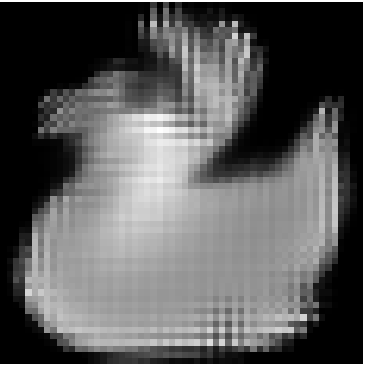}};
\node (n4) [below=of n3.south]          {};
\node (n41) [below=of n31.south]    	{};
\node (n42) [below=of n32.south]        {};
\node (n43) [below=of n33.south]    	{};
\node (n44) [below=of n34.south]        {};
\node (n5) [below=of n4.south]          {};
\node (n52) [left=of n5]                {\includegraphics[width=12mm, height=12mm]{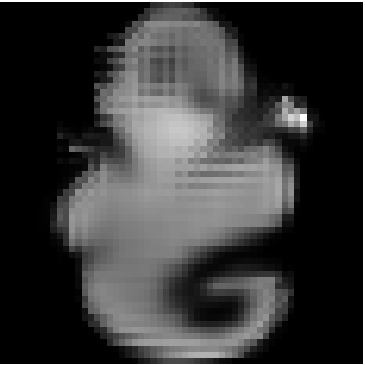}};
\node (n53) [right=of n5]               {\includegraphics[width=12mm, height=12mm]{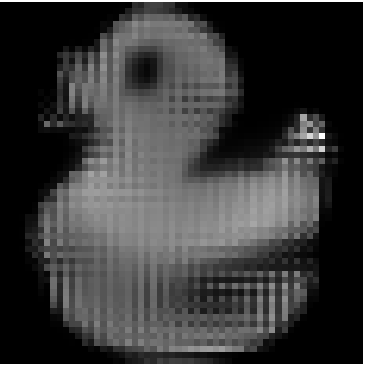}};
\node (n51) [left=of n52]               {\includegraphics[width=30mm, height=20mm]{figs/seahorse.eps}};
\node (n54) [right=of n53]              {\includegraphics[width=30mm, height=20mm]{figs/duck.eps}};
\end{tikzpicture}
\caption{Approximate barycenters obtained by running the multimarginal Sinkhorn (top) and accelerated multimarginal Sinkhorn (bottom) algorithms.}\label{fig:barycenter}
\end{figure}

%!TEX root = paper.tex
\section{Conclusion}\label{sec:conclusion}
In this paper, we have studied the multimarginal optimal transport (MOT) problem, providing new algorithms and complexity bounds for approximating this problem. We demonstrated that the standard linear programming (LP) form of the MOT problem is not a minimum-cost flow problem when $m \geq 3$. This encourages us to study the alternatives to combinatorial algorithms and standard \textit{deterministic} interior-point algorithms. In particular, we considered an entropic regularized version of the MOT problem, developing two \textit{deterministic} algorithms --- the multimarginal Sinkhorn and accelerated multimarginal Sinkhorn algorithms --- for solving it. Combined with a new rounding scheme, the multimarginal Sinkhorn algorithm can solve the MOT problem and achieves a near-linear time complexity bound of $\bigO(m^3 n^m\|C\|_\infty^2\log(n)\varepsilon^{-2})$. For the accelerated multimarginal Sinkhorn algorithm, the complexity bound is $\bigO(m^3 n^{m+1/3}\|C\|_\infty^{4/3}(\log(n))^{1/3}\varepsilon^{-4/3})$ which is not near-linear in the number of variables $n^m$ but has better dependence on $1/\varepsilon$ than that of the multimarginal Sinkhorn algorithm.

We now discuss a few directions that arise naturally from our work. First, the complexity bounds of the proposed algorithms in this paper do not incorporate low-rank approximation framework for the cost tensor $C$. Intuitively, these low-rank approaches will lead to an improvement of these complexity bounds in terms of the number of support points $n$. Therefore, with the low-rank approaches, the implementation of these algorithms will be feasible under the large-scale settings of the MOT problem. Second, as mentioned in the paper, one drawback of the entropic regularization is that the sparsity of the solution is lost. Even though an $\varepsilon$-approximate transportation plan can be obtained efficiently, it is not clear how different the resulting sparsity pattern of the obtained solution is with respect to the solution of the actual MOT problem. An important direction is to incorporate sparsity penalty functions to the entropic regularized MOT problem such that an $\varepsilon$-approximate sparse transportation plan is achieved. Third, the MOT problem suffers from curse of dimensionality, demonstrating the importance of efficient dimension reduction frameworks in both theory and practice. Finally, it is of interest to extend the current algorithms in the paper to the multimarginal optimal transport among general measures, which are not necessarily probability measures, such as multimarginal unbalanced optimal transport~\citep{Pham-2020-Unbalanced} or multimarginal partial optimal transport~\citep{Khang-2022-MPOT}.

\section{Acknowledgments}
This work was supported in part by the Mathematical Data Science program of the Office of Naval Research under grant number N00014-18-1-2764 to MJ, and by the NSF IFML 2019844 award and research gifts by UT Austin ML grant to NH.
%%%%%%%%%%%%%%%%%%%%%%%%%%%%%%%%%%%%%%%%%%%%%%%%%%%%%%%%%%%%%%%%%%%%%

%%% References
\bibliographystyle{plainnat}
\bibliography{ref}
\end{document}